\PassOptionsToPackage{dvipsnames}{xcolor} %
\documentclass{article} %
\usepackage{iclr2024_conference,times}
\usepackage{amsthm}
\usepackage{amsmath}

\usepackage[utf8]{inputenc}
\usepackage[T1]{fontenc}

\usepackage{parskip}
\usepackage[colorlinks=true, linkcolor=Red, urlcolor=Blue, citecolor=Blue, anchorcolor=Blue,backref=page]{hyperref}
\usepackage{url}
\usepackage{booktabs}
\usepackage{microtype}
\usepackage[shortlabels]{enumitem}
\usepackage{nomencl}
\usepackage{algorithm}
\usepackage{algpseudocode}
\usepackage{makecell}
\usepackage{thmtools} 
\usepackage{thm-restate}
\usepackage{wrapfig}
\usepackage{bbm}
\usepackage{colortbl}
\usepackage{multirow}
\usepackage{cleveref}
\usepackage{notation}
\usepackage{subcaption}
\usepackage{pifont}
\usepackage{tabularray}
\UseTblrLibrary{booktabs}
\usepackage[most]{tcolorbox}

\usepackage{adjustbox}
\usepackage{textcomp}

\usepackage{minitoc}
\usepackage{authblk}

\DeclareUnicodeCharacter{3000}{\textcolor{red}{IAMAUNICODEWHITESPACE}}

\newtcolorbox{empheqboxed}{colback=Gray!20, 
 colframe=white,
 width=\textwidth,
 sharpish corners,
 top=1mm, %
 bottom=0pt,
 left=2pt,
 right=2pt
}

\definecolor{Green}{HTML}{17891a}

\crefname{section}{\S}{\S\S}
\crefname{subsection}{\S}{\S\S}
\crefname{subsubsection}{\S}{\S\S}
\crefname{figure}{Fig.}{Figs.}
\crefname{table}{Tab.}{Tabs.}
\crefname{definition}{Defn.}{Defns.}
\crefname{corollary}{Cor.}{Cors.} 
\crefname{proposition}{Proposition}{Propositions}
\crefname{theorem}{Thm.}{Thms.}
\crefname{appendix}{App.}{Apps.}
\crefname{remark}{Remark}{Remarks}
\crefname{principle}{Principle}{Principles}
\crefname{example}{Example}{Examples.}
\crefname{lemma}{Lemma}{Lemmas}
\crefname{claim}{Claim}{Claims}
\crefname{assumption}{Asm.}{Asms.}
\numberwithin{equation}{section}
\numberwithin{assumption}{section}
\everypar{\looseness=-1}

\title{Multi-View Causal Representation Learning with Partial Observability
}

\makeatletter \renewcommand\AB@affilsepx{\, \,   \protect\Affilfont} \makeatother

\author[1,2]{\textbf{Dingling Yao}}
\author[3]{\textbf{Danru Xu}}
\author[7,8]{\textbf{S\'ebastien Lachapelle}}
\author[3,6]{\textbf{Sara Magliacane}}
\author[9]{\textbf{Perouz Taslakian}}
\author[4]{\textbf{Georg Martius}}
\author[2,5]{\textbf{Julius von K\"ugelgen}}
\author[1]{\textbf{Francesco Locatello}}
\affil[1]{Institute of Science and Technology Austria}
\affil[2]{Max Planck Institute for Intelligent Systems, T\"ubingen}
\affil[3]{University of Amsterdam}
\affil[4]{University of T\"ubingen}
\affil[5]{University of Cambridge}
\affil[6]{MIT-IBM Watson AI Lab}
\affil[7]{Samsung - SAIT AI Lab} 
\affil[8]{Mila, Université de Montréal}
\affil[9]{ServiceNow Research}

\newcommand\pcref[1]{(\cref{#1})}

\begin{document}
\iclrfinalcopy

\doparttoc %
\faketableofcontents %

\part{} %

\maketitle
\vspace{-1em}
\begin{abstract}
    We present a unified framework for studying the identifiability of representations learned from simultaneously observed views, such as different data modalities. We allow a \emph{partially observed} setting in which each view constitutes a nonlinear mixture of a subset of underlying latent variables, which can be causally related.
    We prove that the information shared across all subsets of any number of views can be learned up to a smooth bijection using contrastive learning and a single encoder per view. 
    We also provide graphical criteria indicating which latent variables can be identified through a simple set of rules, which we refer to as \textit{identifiability algebra}. Our general framework and theoretical results unify and extend several previous works on multi-view nonlinear ICA, disentanglement, and causal representation learning. We experimentally validate our claims on numerical, image, and multi-modal data sets. Further, we demonstrate that the performance of prior methods is recovered in different special cases of our setup. 
    Overall, we find that access to multiple partial views enables us to identify a more fine-grained representation, under the generally milder assumption of partial observability. %
\end{abstract}
\section{Introduction}
\looseness=-1Discovering latent structure underlying data has been important across many scientific disciplines, spanning neuroscience~\citep{vigario1997independent,brown2001independent}, communication theory~\citep{ristaniemi1999performance,donoho2006compressed}, natural sciences~\citep{wunsch1996ocean,chadan2012inverse,trapnell2014dynamics}, and countless more. The underlying assumption is that many natural phenomena measured by instruments have a simple structure that is lost in raw measurements. In the famous cocktail party problem~\citep{cherry1953some}, multiple speakers talk concurrently, and while we can easily record their overlapping voices, we are interested in understanding what individual people are saying.
From the methodological perspective, such inverse problems became common in machine learning with breakthroughs in linear~\citep{comon1994independent,darmois1951analyse,hyvarinen2000independent} and non-linear~\citep{hyvarinen2019nonlinear} Independent Component Analysis (ICA), and developed into deep learning methods for disentanglement~\citep{bengio2013representation,higgins2017beta}. More recently, approaches to causal representation learning~\citep{scholkopf2021toward} began relaxing the key assumption of independent latents central to prior work (the \textit{independent} in ICA), allowing for and discovering (some) hidden causal relations~\citep{brehmer2022weakly,lippe2022citris,lachapelle2022disentanglement,zhang2023identifiability, ahuja2023interventional, varici2023score,squires2023linear,von2023nonparametric}.

\looseness=-1This problem is often modeled as a two-stage sampling procedure, where latent variables $\zb$ are sampled i.i.d.\ from a distribution $p_\zb$, and the observations $\xb$ are functions thereof. Intuitively, the latent variables describe the causal structure underlying a specific environment, and they are only observed through sensor measurements, entangling them via so-called ``\textit{mixing functions}''. Unfortunately, if these mixing functions are non-linear, the recovery of the latent variables is generally impossible, even if the latent variables are independent~\citep{locatello2019challenging,hyvarinen1999nonlinear}. Following these negative results, the community has turned to settings that relax the i.i.d. condition in different ways. One particularly successful paradigm has been the assumption that data is not independently sampled, and in fact, multiple observations may refer to the same realization of the latent variables. This multi-view setup has generated a flurry of results in ICA~\citep{gresele2019incomplete,zimmermann2021contrastive,pandeva2023multi}, disentanglement~\citep{locatello2020weakly,klindt2021towards,fumero2023leveraging,lachapelle2022synergies,ahuja2022weakly}, and 
causal representation learning~\citep{von2021self,daunhawer2023identifiability,brehmer2022weakly}. 

\looseness=-1This paper provides a unified framework for several identifiability results in observational multi-view causal representation learning under partial observability. We assume that different views need not be functions of \textit{all} the latent variables, but only of some of them. For example, a person may undertake different medical exams, each shedding light on some of their overall health status (assumed constant throughout the measurements) but none offering a comprehensive view. An X-ray may show a broken bone, an MRI
how the fracture affected nearby tissues, and a blood sample may inform about ongoing infections. Our framework also allows for an arbitrary number of views, each measuring partially overlapping latent variables.
It includes multi-view ICA and disentanglement as special cases. 

\looseness=-1More technically, we prove that any shared information across arbitrary subsets of views and modalities can be learned up to a smooth bijection using contrastive learning. Non-shared information can also be identified if it is independent of other latent variables. With a single identifiability proof, our result implies the identifiability of several prior works in causal representation learning~\citep{von2021self,daunhawer2023identifiability}, non-linear ICA~\citep{gresele2019incomplete}, and disentangled representations~\citep{locatello2020weakly,ahuja2022weakly}. In addition to weaker assumptions, our framework retains the algorithmic simplicity of prior contrastive multi-view~\citep{von2021self} and multi-modal~\citep{daunhawer2023identifiability} causal representation learning approaches. Allowing partial observability and arbitrarily many views, our framework is significantly more flexible than prior work, allowing us to identify shared information between all subsets of views and not just their joint intersection.

\begin{wrapfigure}{r}{0.33\textwidth}
    \vspace{-10pt}
    \centering
    \includegraphics[width=0.33\textwidth]{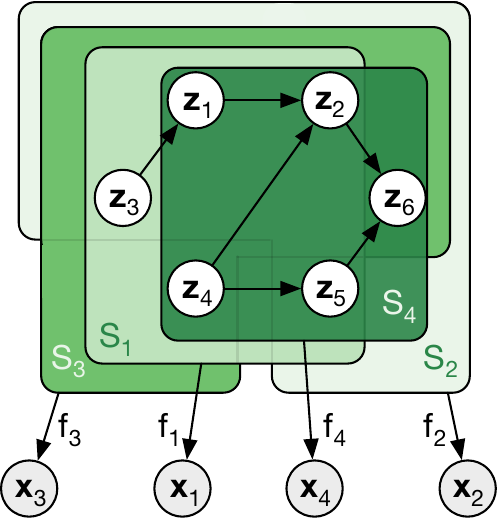}
    \caption{\small \textbf{Multi-View Setting with Partial Observability}, for~\cref{example:intuitive} with $K{=}4$ views and $N{=}6$ latents. Each view $\xb_k$ is generated by a subset $\zb_{S_k}$ of the latent variables through a {view-specific mixing function}~$f_k$.
    Directed arrows between latents indicate causal relations.
    }
    \vspace{-40pt}
    \label{fig:example_graph}
\end{wrapfigure}

We highlight the following contributions:
\begin{enumerate}[leftmargin=*]
    \item We provide a unified framework for identifiability in observational multi-view causal representation learning with partial observability. This generalizes the multi-view setting in two ways: allowing (i) \emph{any arbitrary number of views}, and (ii) partial observability with non-linear mixing functions. We prove that any shared information across arbitrary subsets of views and modalities can be learned up to a smooth bijection using contrastive learning and provide straightforward graphical criteria to categorize which latents can be recovered.
    \item  With a single proof, our result implies the identifiability of several prior works in causal representation learning, non-linear ICA, and disentangled representations as special cases. 
    \item We conduct experiments for various unsupervised and supervised tasks and empirically show that (i) the performance of prior works can be recovered using a special setup of our framework and (ii) our method indicates promising disentanglement capabilities with encoder-only networks.
\end{enumerate}

\section{Problem Formulation}
\label{sec:problem_formulation}
We formalize
the data generating process 
as a latent variable model. 
Let $\zb=(\zb_1, ..., \zb_N)\sim p_\zb$ be possibly dependent (causally related) latent variables taking values in $\Zcal=\Zcal_1\times ... \times \Zcal_N$, where 
$\Zcal \subseteq \RR^{N}$ is an open, simply connected latent space
with associated probability density $p_\zb$. Instead of directly observing $\zb$, we observe a set of entangled measurements or views $\xb := (\xb_1, \dots, \xb_K)$. Importantly, we assume that each observed view $\xb_k$ may only depend on
some of the latent variables, which we call ``\textit{view-specific latents}''~$\zb_{S_k}$, indexed by subsets $S_1, ..., S_K\subseteq[N]=\{1,...,N\}$. 
For any $A \subseteq [N]$, the subset of latent variables $\zb_{A}$ and corresponding latent sub-space $\Zcal_{A}$
are given by:
    \begin{equation*}
    \textstyle
        \zb_{A} := \{\zb_j : j \in A\}, \qquad  \qquad \Zcal_{A} := \bigtimes_{j \in A} \Zcal_j\,.
    \end{equation*}
    Similarly, for any $V \subseteq [K]$, the subset of views $\xb_V$ and corresponding observation space $\Xcal_V$ are: 
    \begin{equation*}
    \label{eq:view_defn}
    \textstyle
        \xb_{V} := \{\xb_k : k \in V\}, \qquad  \qquad \Xcal_{V} := \bigtimes_{k \in V} \Xcal_k\,.
    \end{equation*}
The \emph{view-specific mixing functions} $\{f_k: \Zcal_{S_k} \to \Xcal_k\}_{k \in [K]}$ are smooth, invertible mappings from the view-specific latent subspaces~$\Zcal_{S_k}$ to observation spaces $\Xcal_k \subseteq \RR^{\dimension(\xb_k)}
$
with $\xb_k := f_k(\zb_{S_k})$. Formally, the generative process for the views $\{\xb_1, \dots, \xb_K\}$ is given by:
\begin{equation*}
    \zb \sim p_\zb, \qquad\qquad \xb_k:=f_k(\zb_{S_k}) \quad \forall k \in [K],
\end{equation*}
i.e., each view $\xb_k$ depends on latents $\zb_{S_k}$ through a \emph{mixing function}~$f_k$, as illustrated in~\cref{fig:example_graph}.

\begin{assumption}[General Assumptions]
\label{assmp:general_assumption}
    For the latent generative model defined above:
    \begin{enumerate}[(i),topsep=0em,itemsep=0em]
    \item Each view-specific mixing function $f_k$ is a diffeomorphism;
    \item $p_\zb$ is a smooth and continuous density on $\Zcal$ with $p_\zb > 0$ almost everywhere.
\end{enumerate}
\end{assumption}

\begin{empheqboxed}
\begin{example}
\label{example:intuitive}
Throughout, we illustrate key concepts and results using the following example with $K{=}4$ views, $N{=}6$ latents, and dependencies among the $\zb_k$ shown as a graphical model in~\cref{fig:example_graph}. 

\begin{equation}
\textstyle
\label{eq:example}
\begin{aligned}
     \xb_1 &= f_1(\zb_1, \zb_2, \zb_3, \zb_4, \zb_5),  \qquad
     \xb_2 = f_2(\zb_1, \zb_2, \zb_3, \zb_5, \zb_6), \\
     \xb_3 &= f_3(\zb_1, \zb_2, \zb_3, \zb_4, \zb_6), \qquad
     \xb_4 = f_4(\zb_1, \zb_2, \zb_4, \zb_5, \zb_6) .\quad
\end{aligned}
\end{equation}
\end{example}
\end{empheqboxed}

\looseness=-1 Consider a set~$\xb_{V}$ of %
jointly observed views,  
and let $\Vcal := \{V_i \subseteq V: |V_i| \geq 2\}$ be the set of subsets $V_i\in\Vcal$ indexing two or more views.
For any subset of views $V_i$, we refer to the set of \emph{shared} latent variables (i.e., those influencing each view in the set) as the ``\emph{content}'' or ``\emph{content block}" of $V_i$.
Formally, content variables $\zb_{C_i}$ are obtained as
intersections of view-specific indexing sets:
\begin{equation}
\textstyle
\label{eq:defn_content}
\begin{aligned}
\textstyle
C_{i} = \bigcap_{k \in V_i} S_k\,.
\end{aligned}
\end{equation}
Similarly, for each view $k\in V$, we can define the non-shared (``style") variables as $\zb
_{S_k \setminus C_i}$. We use $C$ and $\zb_{C}$ without subscript to refer to the joint content across all observed views $\xb_V$. 
\begin{empheqboxed}
For~\cref{example:intuitive}, the content of $\xb_{V_1} = (\xb_1,\xb_2)$ is $\zb_{C_1}
= (\zb_1,\zb_2,\zb_3,\zb_5)$; the content for all four views $\xb
= (\xb_1, \xb_2, \xb_3, \xb_4)$ jointly is $\zb_{C} = (\zb_1,\zb_2)$, and the style for $\xb_1$ is $\zb_{S_1 \setminus C} = (\zb_3, \zb_4, \zb_5)$.
\end{empheqboxed}

\looseness=-1
\begin{remark}[``Content-Style'' Terminology]
We adopt these terms from~\citet{von2021self}, but note
that, in our setting, 
they are relative to a specific subset of views.
Unlike in some prior works~\citep{gresele2019incomplete,von2021self,daunhawer2023identifiability}, style variables are generally not considered irrelevant, but may also be of interest and can sometimes be identified by other means (e.g., from other subsets of views or because they are independent of content).
\end{remark}

\looseness=-1 Our goal is to show that we can simultaneously identify multiple content blocks given a set of jointly observed views under weak assumptions. This extends previous work~\citep{gresele2019incomplete,von2021self,daunhawer2023identifiability} where only one block of content variables is considered. Isolating the shared content blocks from the rest of the view-specific style information, the learned representation (estimated content) can be used in downstream pipelines, such as classification tasks~\citep{lachapelle2022synergies,fumero2023leveraging}. In the best case, if each latent component is represented as one individual content block, we can learn a fully disentangled representation~\citep{higgins2018definition,locatello2020weakly,ahuja2022weakly}. To this end, we restate the definition of \emph{block-identifiability}~\citep[][Defn 4.1]{von2021self} for the multi-modal, multi-view setting:
\begin{definition}[Block-Identifiability]
    \label{defn:block-identifiability}
    The true content variables $\cb$
    are \emph{block-identified} by a function $g: \Xcal \to \RR^{\dimension(\cb)}$ if the inferred content partition $\hat{\cb}= g(\xb)$ contains \emph{all and only} information about $\cb$%
    , i.e., if there exists some smooth \emph{invertible} mapping $h: \RR^{\dimension(\cb)} \to \RR^{\dimension(\cb)}$ s.t. $\hat{\cb}= h(\cb)$.
\end{definition}

Note that the inferred content variables $\hat{\boldsymbol{c}}$ can be a set of \emph{entangled} latent variables rather than a single one. This differentiates our paper from the line of work on \emph{disentanglement}~\citep{locatello2020weakly,fumero2023leveraging,lachapelle2022synergies}, which seek
to disentangle \emph{individual} latent factors and can thus be considered as special cases of our framework with content block sizes equal to one.

\section{Identifiability Theory}
\label{sec:identifiability}
\textbf{High-Level Overview. }
This section presents a unified framework for studying identifiability from multiple partial views: 
we start by establishing identifiability of the shared content block $\zb_{C}$ from \emph{any number of partially observed views}~\pcref{thm:ID_from_sets}. 
The downside of this approach is that if we seek to learn content from different subsets, we need to train an \textit{exponential} number of encoders for the same modality, one for each subset of views. 
We, therefore, extend this result and show that by considering any subset of the jointly observed views, various blocks of content variables can be identified by \emph{one single} view-specific encoder~\pcref{thm:general_ID_from_sets_size_unknown}. 
After recovering multiple content blocks \emph{simultaneously}, we show in~\cref{cor:id_alg_intersection,cor:id_alg_complement,cor:id_alg_union} that a qualitative description of the data generative process such as in~\cref{fig:example_graph} can be sufficient to determine exactly the extent to which individual latents or groups thereof can be identified and disentangled. Full proofs are included in~\cref{sec:proofs}.

\begin{definition}[Content Encoders]
\label{defn:content_encoders}
    Assume that the content size $|C|$ is given for any jointly observed views $\xb_{V}$. The content encoders $G:= \{g_k:\Xcal_k\to (0, 1)^{|C|}\}_{k \in V}$ consist of smooth functions mapping from the respective observation spaces to the $|C|$-dimensional unit cube.
\end{definition}

\begin{restatable}[Identifiability from a \emph{Set} of Views]{theorem}{idset}
\label{thm:ID_from_sets}
Consider a set of views $\xb_{V}$ satisfying~\cref{assmp:general_assumption}, and let $G$ be a set of content encoders~\pcref{defn:content_encoders} that minimizes the following objective    
\begin{equation}
    \label{eq:main_loss_set}
    \Lcal 
    \left(G\right)=
    \underbrace{\textstyle\sum_{\substack{k< k^\prime \in V
    }}
    \EE
    \left[\norm{g_{k}(\xb_{k})-g_{k^\prime}(\xb_{k^\prime})}_2\right]}_{\text{\textcolor{Green}{Content alignment}}}- 
    \underbrace{\textstyle\sum_{k \in V} H\left(g_{k}(\xb_k)\right)}_{\text{\textcolor{Blue}{Entropy regularization}}},
\end{equation}
where the expectation is taken w.r.t.\ $p(\xb_V)$ and $H(\cdot)$ denotes differential entropy.
Then the shared \textbf{content} variable $\zb_{C} := \{\zb_j : j \in C\}$ is block-identified ~\pcref{defn:block-identifiability} by $g_{k} \in G$ for any $k\in V$.
\end{restatable}

\begin{empheqboxed}
\looseness=-1\textbf{Intuition.} 
The \emph{\textcolor{Green}{alignment}} enforces the content encoders $g_k$ only to encode content and discard styles, while the maximized \emph{\textcolor{Blue}{entropy}} implies uniformity and thus invertibility. For~\cref{example:intuitive}, recall that the joint content 
is $C = \cap_{k \in [4]} \{S_k\} = \{1, 2\}$. \Cref{thm:ID_from_sets} then states that, for each %
$k=1, 2, 3, 4$, the content encoders $G=\{g_k: \Xcal_k \to (0, 1)^{|C|}\}$ which minimize the loss in~\cref{eq:main_loss_set} are actually invertible mappings of the ground truth content $\{\zb_1, \zb_2\}$, i.e., 
    $g_k(\xb_k) = h_k(\zb_1, \zb_2)$
for smooth invertible functions $h_k: \Zcal_1 \times \Zcal_2 \to (0, 1)^2$.
\end{empheqboxed}

\looseness=-1\textbf{Discussion.} \Cref{thm:ID_from_sets} provides a learning algorithm to infer \emph{one} jointly shared content block for \emph{all} observed views in a set, extending prior results that only consider two views~\citep{von2021self,daunhawer2023identifiability,locatello2020weakly}. However, to discover another content block $C_i$ w.r.t.\ a subset of views $V_i \subset V$ as defined in~\cref{sec:problem_formulation},  we need to train another set of encoders, since the dimensionality of the content might change. 
Ideally, we would like to learn \emph{one view-specific encoder}~$r_k$ that maps from the observation space $\Xcal_k$ to some $|S_k|$-dimensional manifold and can block-identify all shared contents $\zb_{C_i}$ using \emph{one} training run, combined with separate \textit{content selectors}. 
\begin{definition}[View-Specific Encoders]
\label{defn:view_specific_encoders}
    The \emph{view-specific encoders} $R := \{r_k:\Xcal_k\to \Zcal_{S_k}\}_{k \in V}$ consist of smooth functions mapping from the respective observation spaces to the view-specific latent space, where the dimension of the %
    $k$\textsuperscript{th} latent space $|S_k|$ is assumed known for all $k \in V$.
\end{definition}
\begin{empheqboxed}
    \looseness=-1 \textbf{Intuition.} 
    The view-specific encoders learn \emph{all view-related content blocks simultaneously}, instead of training a combinatorial number of networks (as would be implied by~\cref{thm:ID_from_sets}). 
    The view-specific encoders should learn not only a single block of content variables, but instead learn to recover \textit{all} shared latents in a way that makes it \textit{easy to extract various different content blocks using simple readout functions}. This is possible by construction, e.g.,  if each $r_k$ learns to invert the ground truth mixing $f_k$.
    Inspired by this idea, we introduce \emph{content selectors}.
\end{empheqboxed}
\begin{definition}[Selection]
\label{defn:selection}
    A selection $\oslash$ operates between two vectors $a \in \{0, 1\}^d\, , b \in \RR^d$ s.t.
    \begin{equation*}
        \textstyle
        a \oslash b := [b_j: a_j = 1, j \in [d]]
    \end{equation*}
\end{definition}
\begin{definition}[Content Selectors]
\label{defn:content_selectors} The content selectors $\Phi:= \{\phi^{(i, k)}\}_{V_i \in \Vcal, k \in V_i}$ with $\phi^{(i, k)} \in\{0,1\}^{|S_k|}$ perform selection~\pcref{defn:selection} on the encoded information: for any subset $V_i$ and view $k \in V_i$ we have the selected representation:
$
    \textstyle
    \phi^{(i, k)} \oslash \hat{\zb}_{S_k} = \phi^{(i, k)} \oslash r_k(\xb_k),
$
with $\norm{\phi^{(i, k)}}_0 = \norm{\phi^{(i, k^\prime)}}_0$ for all $V_i \in \Vcal, k, k^\prime \in V_i$.
\end{definition}
\begin{empheqboxed}
\textbf{Intuition.} 
Using \emph{learnable} binary weights $\Phi$, the content selectors $\phi^{(i, k)}$ should 
pick out those latents among the representation $\hat{\zb}_{S_k}$ extracted by $r_k$ that belong to the content block $C_i$ shared among $V_i$. 
For~\cref{example:intuitive}, consider a learned representation $r_1(\xb_1) = (\hat{\zb}_1, \hat{\zb}_2, \hat{\zb}_3, \hat{\zb}_4, \hat{\zb}_5)$. Applying a content selector with weight $\phi^{(i=1, k=1)} = [1, 1, 1, 0, 1]$ then yields: $(\hat{\zb}_1, \hat{\zb}_2, \hat{\zb}_3,\hat{\zb}_5)$.
\end{empheqboxed}
\looseness=-1\textbf{What is missing? } 
While aligning various content blocks based on the same representation $r_k(\xb_k)$ should promote \emph{disentanglement}, maximizing the {\textcolor{Blue}{entropy}} $H(r_k(\xb_k))$ of the learned representation  (as in~\cref{thm:ID_from_sets}) promotes \emph{uniformity}.
The latter implies invertibility of the encoders~\citep{zimmermann2021contrastive}, which is necessary for block-identifiability~\pcref{defn:block-identifiability}. However, since a uniform representation has independent components by definition, disentanglement and uniformity cannot be achieved simultaneously unless all ground truth latents are mutually independent (a strong assumption we are not willing to make).
Thus, to \emph{theoretically} achieve invertibility while preserving disentanglement, we introduce a set of auxiliary \emph{projection} functions.
\begin{definition}[Projections]
\label{defn:aux_transformations}
    The set of projections $T := \{t_k\}_{k \in V}$ 
    consist of functions $t_k: \Zcal_{S_k} \to (0, 1)^{|S_k|}$ mapping each view-specific latent space to a hyper unit-cube of the same dimension $|S_k|$.
\end{definition}
\begin{empheqboxed}
\looseness=-1\textbf{Intuition.}
The projection functions can be understood as mathematical tools: %
by maximizing the entropy and thus enforcing uniformity of 
 \textit{projected representations} $t_k\circ r_k(\xb_k)$, we can show that $r_k$ needs to be invertible without interfering with the disentanglement of different content blocks.
\end{empheqboxed}

\textbf{What if the content dimension is unknown? }
In~\cref{thm:ID_from_sets} we assumed that
the size $|C|$ of the 
shared content block is known, and the encoders map to a space
of dimension $|C|$. 
In the following, we do \textit{not} assume that the content size is given.
Instead,
we will show
that the correct content block can still be discovered by ensuring that 
as much information as possible is shared across any given subset of views.
To this end, we define the following information-sharing regularizer.
\begin{definition}[Information-Sharing Regularizer]
\label{defn:info_sharing_reg}
    The following regularizer penalizes the $L_0$-norm~$\norm{\cdot}_0$ of the content selectors $\Phi$:
    $
        \textstyle
        \label{eq:reGax_information_sharing}
        \mathrm{Reg}(\Phi) := - \sum_{V_i \in \Vcal} \sum_{k \in V_i} 
        \norm{\phi^{(i, k)}}_0\,.
    $
\end{definition}
\begin{empheqboxed}
\looseness=-1 \textbf{Intuition.}
$\mathrm{Reg}(\Phi)$ sums the number of shared latents over $V_i \subseteq \Vcal$ and  $k \in V_i$. It decreases when $\phi^{(i, k)}$ contains more ones, i.e., more latents are shared across views $k \in V_i$. 
Thus, 
$\mathrm{Reg}(\Phi)$ encourages the encoders to reuse the learned latents 
and 
maximize the shared information content. 
\end{empheqboxed}
\begin{restatable}[View-Specific Encoder for Identifiability]{theorem}{unifEnc}
\label{thm:general_ID_from_sets_size_unknown}
Let $R, \Phi, T$ respectively be any view-specific encoders (\cref{defn:view_specific_encoders}), content selectors (\cref{defn:content_encoders}) and projections (\cref{defn:aux_transformations}) that solve the following constrained optimization problem:
\begin{equation}
\label{eq:main_loss_set_size_unknown}
\min\, \mathrm{Reg}(\Phi) \qquad  \text{subject to:} \qquad  R, \Phi, T \in \argmin_{} \,\Lcal \left(R, \Phi, T \right)
\end{equation}
where
\begin{equation}\label{eq:main_loss_set_size_unknown_no_reg}
    \textstyle
    \begin{aligned}
    \Lcal
    \left(R, \Phi, T\right)
    =&
    \sum_{V_i \in \Vcal} \sum_{\substack{k, k^\prime \in V_i \\ k < k^\prime}} \underbrace{
    \EE
    \left[\norm{\phi^{(i, k)} \oslash r_{k}(\xb_{k}) - \phi^{(i, k^{\prime})} \oslash r_{k^\prime}(\xb_{k^\prime})}_2\right]}_{\text{\textcolor{Green}{Content alignment}}}
    - 
     \sum_{k \in V} \underbrace{H\left(t_k \circ r_{k}(\xb_k)\right)}_{\text{\textcolor{Blue}{Entropy}}},
    \end{aligned}
  \end{equation}
Then for any subset of views $V_i \in \Vcal$ and any view $k \in V_i$\,, $\phi^{(i, k)} \oslash r_k$ block-identifies (\cref{defn:block-identifiability}) the shared \textbf{content} variables $\zb_{C_i}$, as defined in~\cref{eq:defn_content}. 
\end{restatable}

\begin{empheqboxed}
\looseness=-1\textbf{Intuition.} 
For~\cref{example:intuitive}, the representation for $\xb_1$ obtained by minimizing~\cref{eq:main_loss_set_size_unknown} is given by
$\hat{\zb} := r_1(\xb_1) = r_1(f_1(\zb_{S_1})) = r_1 \circ f_1(\zb_1, \zb_2, \zb_3, \zb_4, \zb_5)$.
Consider the following two subsets of views $V_1,V_2 \in \Vcal$ 
containing $\xb_1$, but sharing different content blocks $C_1, C_2$:
\begin{equation*}
    \begin{aligned}
        \xb_{V_1} &= \{\xb_1, \xb_2\}\,, &\zb_{C_1} = \{\zb_1, \zb_2, \zb_3, \zb_5\}\,, \quad
        \xb_{V_2} &= \{\xb_1, \xb_3\}\,, &\zb_{C_2} = \{\zb_1, \zb_2, \zb_3, \zb_4\} \,.
    \end{aligned}
\end{equation*}
Then one of the optimal solutions of the selectors learned by~\cref{thm:general_ID_from_sets_size_unknown} could be
\begin{equation*}
\textstyle
    \phi^{(i=1, k=1)} = [1, 1, 1, 0, 1]\,, \qquad \qquad
    \phi^{(i=2, k=1)} = [1, 1, 1, 1, 0].
\end{equation*}
Hence, the composed results of the selectors and the view-specific encoder $r_1$ give:
\begin{equation*}
\textstyle
    \phi^{(i=1, k=1)} \oslash \hat{\zb} = h_{i=1, k = 1}(\zb_1, \zb_2, \zb_3, \zb_5)\,, \qquad 
    \phi^{(i=2, k=1)} \oslash \hat{\zb} = h_{i=2, k = 1}(\zb_1, \zb_2, \zb_3, \zb_4)
\end{equation*}
where $h_{i, k=1}$ is some smooth bijection, for both $i = 1, 2$.
\end{empheqboxed}

\looseness=-1\textbf{Discussion.} 
Note that~\Cref{eq:main_loss_set_size_unknown} can be rewritten as a regularized loss $\Lcal_{\mathrm{Reg}}
    \left(R, \Phi, T\right)
    =\Lcal
    \left(R, \Phi, T\right)+ \alpha \cdot \mathrm{Reg}(\Phi)$ with a sufficiently small regularization coefficient $\alpha \geq 0$. Overall, \Cref{thm:general_ID_from_sets_size_unknown} further weakens the assumptions of~\cref{thm:ID_from_sets} in
that no content size is required. However, minimizing the information-sharing regularizer is highly non-convex, and having only a finite number of samples makes finding the global optimum challenging. In practice, we could use \emph{Gumbel Softmax}~\citep{jang2017categorical} for unsupervised learning, and consider content sizes as hyper-parameters or follow the approach by~\citet{fumero2023leveraging} for supervised classification tasks. Empirically, we will see that some of the requirements that are needed in theory can be realistically dropped, and different approximations are possible, e.g., incorporating problem-specific knowledge.

\looseness=-1 After discovering various content blocks, we are further interested in how to infer more information from the learned content blocks. 
For example, can we identify $\zb_{C_3} := \{\zb_1, \zb_2, \zb_3\} = \zb_{C_1 \cap C_2}$? This perspective motivates our next results, which focus on how to infer new information based on the previously identified blocks: \textbf{\emph{Identifiability Algebra}}.

Let $\zb_{C_1}, \zb_{C_2}$ with $C_1, C_2 \subseteq [N]$ be two identified blocks of latents.
Then it holds for $C_1, C_2$ that:
\begin{restatable}[Identifiability Algebra: Intersection]{corollary}{idAlgIntersection}
\label{cor:id_alg_intersection}
    The intersection $\zb_{C_1 \cap C_2}$ can be block-identified.
\end{restatable}

\begin{restatable}[Identifiability Algebra: Complement]{corollary}{idAlgComplement}
\label{cor:id_alg_complement}
    If $C_1 {\cap} C_2$ is independent of $C_1 {\setminus} C_2$, then the complement $\zb_{C_1 \setminus C_2}$ can be block-identified.
\end{restatable}

\begin{restatable}[Identifiability Algebra: Union]{corollary}{idAlgUnion}
\label{cor:id_alg_union}
    If $C_1 {\cap} C_2$, $C_1{\setminus} C_2$ and $C_2 {\setminus}C_1 $ are mutually independent, then the union $\zb_{C_1 \cup C_2}$ can be block-identified.
\end{restatable}
\looseness=-1\textbf{Discussion.} While~\cref{cor:id_alg_intersection} refines the identified block of information into smaller intersections, ~\cref{cor:id_alg_complement,cor:id_alg_union} allows to extract ``style'' variables as defined w.r.t.~some specific views, under the assumption that they are independent of the content block, as discussed by~\citet{lyu2021understanding}. However, our setup is more general, as we can not only explain the independent style variables between pairs of \emph{observations}, but also between \emph{learned content representations}. 
Thus, by iteratively applying~\cref{cor:id_alg_complement} we can generalize the statement to any number of identified content blocks.
Combining~\cref{cor:id_alg_intersection,cor:id_alg_complement,cor:id_alg_union} we can immediately tell which part can be block-identified from a set of views $\Vcal$, given a graphical model representation such as~\cref{fig:example_graph} and subject to technical assumptions underlying our main results. Applying~\cref{cor:id_alg_intersection,cor:id_alg_complement,cor:id_alg_union} iteratively on identified blocks can possibly \emph{disentangle} each individual factors of variation, providing a novel approach for disentanglement. If all variables can be isolated up to element-wise nonlinear transformations, we can learn the causal relations by assuming the original link are nonlinear with additive noise. This exactly corresponds to a post-nonlinear model, whose graphic structure can be further identified using causal discovery algorithms~\citep{zhang2006extensions,zhang2009identifiability}.

\section{Related Work and Special Cases of Our Theory}
\label{sec:related_work}
Our framework unifies several prior work, including \emph{multi-view nonlinear ICA}~\citep{gresele2019incomplete}, \emph{weakly-supervised disentanglement}~\citep{locatello2020weakly,ahuja2022weakly} and \emph{content-style identification}~\citep{von2021self,daunhawer2023identifiability}. ~\Cref{tab:summary_work} shows a summarized (non-exhaustive) list of related works and their respective graphical models that can be considered as %
special cases. 
The graphical setups of the individual works can be recovered from our framework~\pcref{fig:example_graph} by varying the number of observed views and causal relations.

\begin{table}[t]
\caption{A non-exhaustive summary of \emph{special cases} of our theory and their graphical models. An asterisk ($^*$) indicates works that have view-specific latents that are not of interest for identifiability.
}
\centering
\newcommand{\cmark}{\ding{51}}%
\newcommand{\xmark}{\ding{55}}%
\scriptsize
\renewcommand{\arraystretch}{2}%
\resizebox{\columnwidth}{!}{%
\begin{tabular}{l m{0.2\linewidth}cccc}
\toprule
\thead{\bf\scriptsize Method} & \thead{\bf\scriptsize Graph} & \thead{\bf\scriptsize Dependent Latents
}  & \thead{\bf\scriptsize Multi-Modal} &  \thead{\bf\scriptsize Partial Observability} & \thead{\bf\scriptsize $> 2$ Views}\\
\midrule 
\multirow{2}{*}{}
\citet{scholkopf2016modeling} & 
\multicolumn{1}{c}{\includegraphics[scale=0.32]{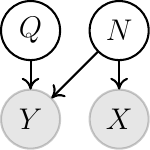}} &
\xmark &  
\cmark & 
\cmark & 
\xmark\\
\multirow{2}{*}{}
\citet{gresele2019incomplete} &  
\multicolumn{1}{c}{\includegraphics[scale=0.32]{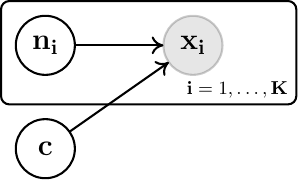}}          &
\xmark & 
\cmark & 
\xmark$^*$ & 
\cmark\\
\multirow{2}{*}{}
\begin{tabular}[c]{@{}l@{}}\citet{locatello2020weakly} \end{tabular}& 
\multicolumn{1}{c}{\includegraphics[scale=0.32]{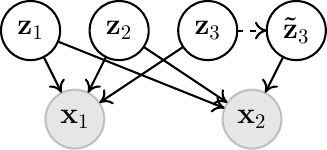}}         &
\xmark &
\xmark & 
\xmark & 
\xmark\\
\multirow{2}{*}{}
\citet{ahuja2022weakly}& 
\multicolumn{1}{c}{\includegraphics[scale=0.32]{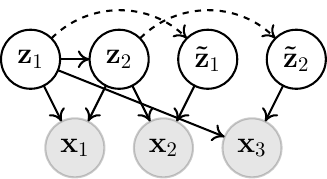}}         &
\cmark  &
\xmark & 
\xmark & 
\cmark\\
\multirow{2}{*}{}
\begin{tabular}[c]{@{}l@{}}
\citet{von2021self} \end{tabular}& 
\multicolumn{1}{c}{\includegraphics[scale=0.32]{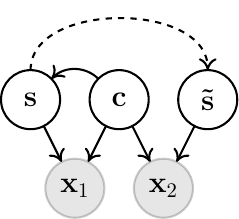}}&
\cmark & 
\xmark & 
\xmark & 
\xmark\\
\multirow{2}{*}{}
\citet{daunhawer2023identifiability} & 
\multicolumn{1}{c}{\includegraphics[scale=0.32]{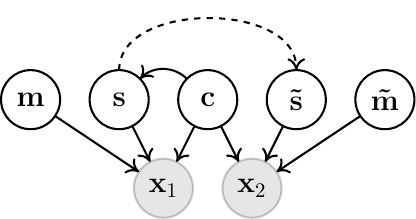}} &
\cmark &  
\cmark & 
\xmark$^*$ & 
\xmark\\
\multirow{2}{*}{}
Ours & 
\multicolumn{1}{c}{\cref{fig:example_graph}}& 
\cmark  &
\cmark & 
\cmark & 
\cmark\\
\bottomrule
\end{tabular}
}
\vspace{-10pt}
\label{tab:summary_work}
\end{table}

\looseness=-1In addition, we present a short overview of other related work which can be connected with our theoretical results, including \emph{causal representation learning}~\citep{sturma2023unpaired,silva2006learning,adams2021identification,kivva2021learning,cai2019triad,xie2020generalized,xie2022identification,morioka2023causal,morioka2023connectivity},
\emph{mutual information-based contrastive learning}~\citep{tian2020makes,tsai2020self,tosh2021contrastive}, 
\emph{latent correlation maximization}~\citep{andrew2013deep,benton2017deep,lyu2020nonlinear,lyu2021understanding}, \emph{nonlinear ICA without auxiliary variables}~\citep{willetts2021don} and \textit{multitask disentanglement with sparse classifiers}~\citep{lachapelle2022synergies,fumero2023leveraging}. Further discussion is given in~\cref{sec:extended_related_work}. 
We remark that several approaches here consider the setting where two observations are generated through an intervention on some latent variable(s). This is sometimes written in the graphical model as two nodes connected by an arrow (shown in the graphs in~\cref{tab:summary_work} as dashed lines $\dashrightarrow$) indicating the pre- and post-intervention versions of the same variable(s).
We stress that this does \textit{not} constitute an example of partial observability. In our setting, latent variables can be simply unobserved, regardless of whether or not they were subject to an intervention.

\textbf{Causal representation learning. }
In the context of causal representation learning (CRL), \citet{sturma2023unpaired} also explicitly consider partial observability in a \emph{linear, multi-domain} setting. 
Several other works on \textit{linear} CRL from \textit{i.i.d.}\ data could also be viewed as assuming partial observability, since they often rely on graphical conditions which enforce each measured variable to depend on a single (a ``pure'' child) or only a few latents~\citep{silva2006learning,adams2021identification,kivva2021learning,cai2019triad,xie2020generalized,xie2022identification}.
In our framework, each view $\xb_k$ instead constitutes a \textit{nonlinear} mixture of \textit{several} latents.
Merging partially observed causal structure has been studied without a representation learning component by~\citet{gresele2022causal,mejia2022obtaining,guo2023out}.

\looseness=-1\textbf{Mutual Information-based Contrastive Learning.}
\Citet{tian2020makes,tsai2020self,tosh2021contrastive} empirically showcase the success of contrastive learning in extracting task-related information across multiple views,
if the augmented views are
redundant to the original data regarding task-related information~\citep{tian2020makes}. From this point of view, the \emph{redundant} task-information can be interpreted as shared content between the views, for which our theory~\pcref{thm:ID_from_sets,thm:general_ID_from_sets_size_unknown} may provide theoretical explanations for the improved performance in downstream tasks.

\looseness=-1\textbf{Latent Correlation Maximization.}
Prior work~\citep{andrew2013deep,benton2017deep,lyu2020nonlinear,lyu2021understanding} showed that maximizing the correlation between the learned representation is equivalent to our \emph{\textcolor{Green}{content alignment}} principle~\pcref{eq:main_loss_set}. The additional invertibility constraint on the learned encoder in their setting is enforced by entropy regularization~\pcref{eq:main_loss_set}, as explained by~\citet{zimmermann2021contrastive}. However, their theory is limited to pairs of views and full observability, while we generalize it to any number of partially observed views.

\looseness=-1\textbf{Nonlinear ICA without Auxiliary Variables.} \citet{willetts2021don} shows nonlinear ICA problem can be solved using non-observable, learnable, clustering task variables $u$, to replace the observed  auxiliary variable in other nonlinear ICA approaches~\citep{hyvarinen2019nonlinear}. While we explicitly require the learned representation to be aligned in a continuous space within the content block, \citet{willetts2021don} impose a \emph{soft} alignment constraint to encourage the encoded information to be similar within a cluster. In practice, the soft alignment requirement can be easily coded in our framework by relaxing the \emph{\textcolor{Green}{content alignment}} with an equivalence class in terms of cluster membership.

\looseness=-1\textbf{Multi-task Disentanglement with Sparse Classifiers.}
Our setup is slightly different from that of~\citet{lachapelle2022synergies,fumero2023leveraging} as they focus on multiple classification tasks using shared encoding and sparse linear readouts. Their sparse classifier head jointly enforces the sufficient representation (regarding the specific classification task, while we aim for the invertibility of the encoders) and a soft alignment up to a linear equivalence class (relaxing our hard alignment). 
However, the identifiability principles we use are similar: sufficient representation (entropy regularization), 
alignment 
and information sharing.
While our results can be easily extended to allow for alignment up to a linear equivalence class, their identifiability theory crucially only covers independent latents.
\section{Experiments}
\label{sec:experiment}
First, we validate~\cref{thm:ID_from_sets,thm:general_ID_from_sets_size_unknown} using numerical simulations in a \emph{fully controlled} synthetic setting. 
Next, 
we conduct experiments on visual (and text) data demonstrating different special cases that are unified by our theoretical framework~\pcref{sec:related_work} and how we extend them. %
We use InfoNCE~\citep{oord2019representation} and BarlowTwins~\citep{zbontar2021barlow} to estimate~\cref{eq:main_loss_set,eq:main_loss_set_size_unknown}. 
The \textcolor{Green}{\emph{content alignment}} is computed by the numerator 
(positive pairs) in InfoNCE
and the \textcolor{Blue}{\emph{entropy regularization}} 
is estimated by the denominator 
(negative pairs). 
Further remarks on contrastive learning and entropy regularization are in~\cref{sec:extended_discussion}. 
For the evaluation, we follow a standard evaluation protocol~\citep{von2021self} and predict the ground truth latents from the learned representation $g_k(\xb_k)$, 
using kernel ridge regression for \emph{continuous} latent variables, 
and logistic regression for \emph{discrete} ones, respectively.
Then, we report the \emph{coefficient of determination} $R^2$ to show the correlation between the learned and ground truth latent variables. 
An $R^2$ close to one between the learned and ground truth variables means that the learned variables are modelling correctly the ground truth, 
indicating block-identifiability~\pcref{defn:block-identifiability}. 
However, $R^2$ is limited as a metric, since any style variable that strongly depends on a content variable would also become predictable, thus showing a high $R^2$ score. 
\begin{wrapfigure}{r}{0.5\linewidth}
    \vspace{-1.5em}
    \centering
    \includegraphics[width=.5\textwidth]{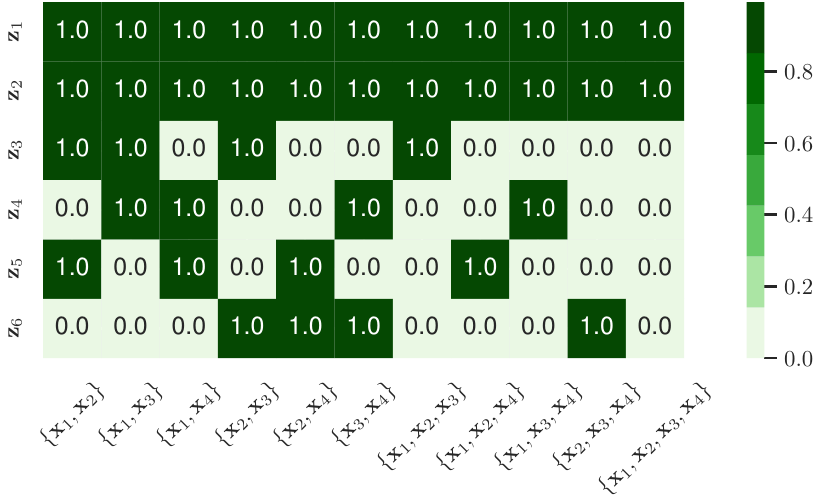}
    \vspace{-2.0em}
    \caption{\small \textbf{Theory Validation}: Average $R^2$ across multiple views generated from \emph{independent} latents.}
    \vspace{-13pt}
    \label{fig:numerical_heatmap_ind}
    \vspace{-7pt}
\end{wrapfigure}
\subsection{Numerical Experiment: Theory validation}
\label{subsec: num_results}
\looseness=-1\textbf{Experimental Setup.}
We generate synthetic data following~\cref{eq:example}. The latent variables are sampled from a Gaussian distribution $\zb \sim \Ncal(0, \Sigma_{\zb})$, where possible \emph{causal} dependencies can be encoded through $\Sigma_{\zb}$. The \emph{view-specific mixing functions} $f_k$ are implemented by randomly initialized invertible MLPs for each view $k \in \{1, \dots 4\}$. 
We report here the $R^2$ scores for the case of \emph{independent} variables, because it is easier to interpret than the $R^2$ scores in the \emph{causally dependent} case, for which we show that the learned representation still contains \emph{all and only} the content information in \cref{app:numerical}.

\looseness=-1\textbf{Discussion.}
\cref{fig:numerical_heatmap_ind} shows how the averaged $R^2$ changes when including more views, 
with the y-axis denoting the ground truth latents and the x-axis showing learned representation from different subsets of views. 
As shown in~\cref{example:intuitive} and~\cref{fig:example_graph}, the content variables are \emph{consistently} identified, having $R^2 \approx 1$, while the \emph{independent} style variables are non-predictable ($R^2 \approx 0$). This numerical result shows that the learned representation explains almost all variation in the content block but nothing from the independent styles, which validates~\cref{thm:ID_from_sets,thm:general_ID_from_sets_size_unknown}.

\begin{wraptable}{r}{0.4\linewidth}
    \centering
    \vspace{-2em}
    \caption{\small \textbf{Self-Supervised Disentanglement Performance Comparison} on \emph{MPI-3D complex}~\citep{gondal2019transfer} and \emph{3DIdent}~\citep{zimmermann2021contrastive}, between our method and Ada-GVAE~\citep{locatello2020weakly}.}
    \label{tab:my_label}
    \vspace{-0.5em}
    \scriptsize
    \begin{tabular}{ccc}
    \toprule
    \rowcolor{Gray!20}  & \multicolumn{2}{c}{\textbf{DCI disentanglement} $\uparrow$} \\
    \rowcolor{Gray!20}  & \scriptsize \textbf{MPI3D complex} & \scriptsize \textbf{3DIdent}\\
    \midrule
    Ada-GVAE & $0.11\pm$ \tiny 0.008  &  $0.09\pm$ \tiny 0.019\\
    Ours & $\mathbf{0.42}\pm$ \tiny 0.020 & $\mathbf{0.30}\pm$ \tiny 0.04\\
    \bottomrule
    \end{tabular}
     \vspace{-13pt}
\end{wraptable}
\subsection{Self-Supervised Disentanglement}
\label{subsec:case1}
\looseness=-1\textbf{Experimental Setup.}
We compare our method (\cref{thm:general_ID_from_sets_size_unknown}) 
with Ada-GVAE~\citep{locatello2019challenging}, 
on \emph{MPI-3D complex}~\citep{gondal2019transfer} 
and 
\emph{3DIdent}~\citep{zimmermann2021contrastive} image datasets. 
We did not compare with \citet{ahuja2022weakly}, since their method needs to know which latent is perturbed, even when guessing the offset. 
We experiment on a pair of views $(\xb_1, \xb_2)$ 
where the second view $\xb_2$ is obtained by randomly perturbing a subset of latent factors of $\xb_1$, 
following~\citep{locatello2019challenging}. 
We provide more details about the datasets and the experiment setup in~\cref{app:ss_dis}. 
As shown in~\cref{tab:my_label}, 
our method outperformed the autoencoder-based Ada-GVAE~\citep{locatello2020weakly}, using only an encoder and computing contrastive loss in the latent space. 

\looseness=-1\textbf{Discussion.}
As both methods are theoretically identifiable, 
we hypothesize that the improvement comes from avoiding reconstructing the image, 
which is more difficult on visually complex data. 
This hypothesis is supported by the fact that self-supervised contrastive learning has far exceeded the performance of autoencoder-based representation learning 
in both classification tasks~\citep{chen2020simple,caron2021emerging,oquab2023dinov2} and object discovery~\citep{seitzer2023bridging}.

\begin{figure}
\centering
\vspace{-10pt}
\begin{subfigure}{.42\textwidth}
  \centering
  \includegraphics[width=\linewidth]{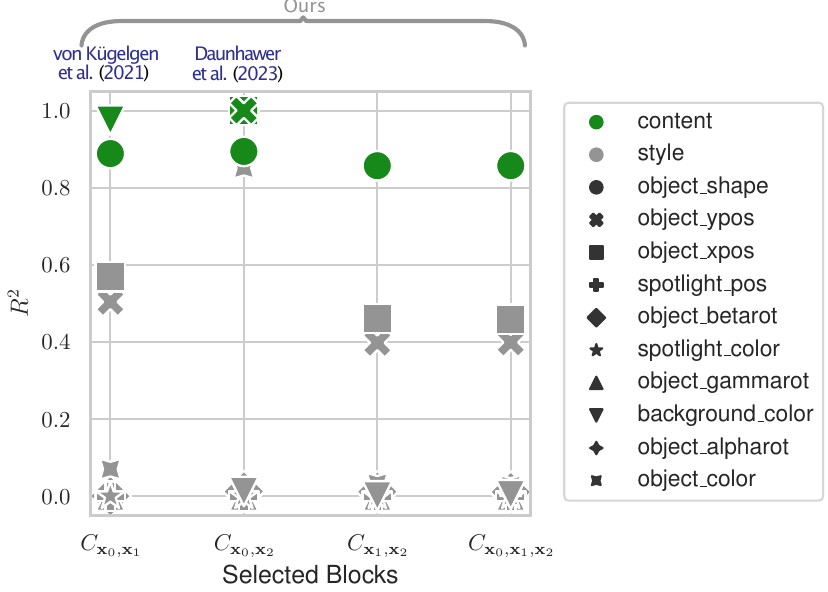}
\end{subfigure}%
\begin{subfigure}{.42\textwidth}
  \centering
  \includegraphics[width=\linewidth]{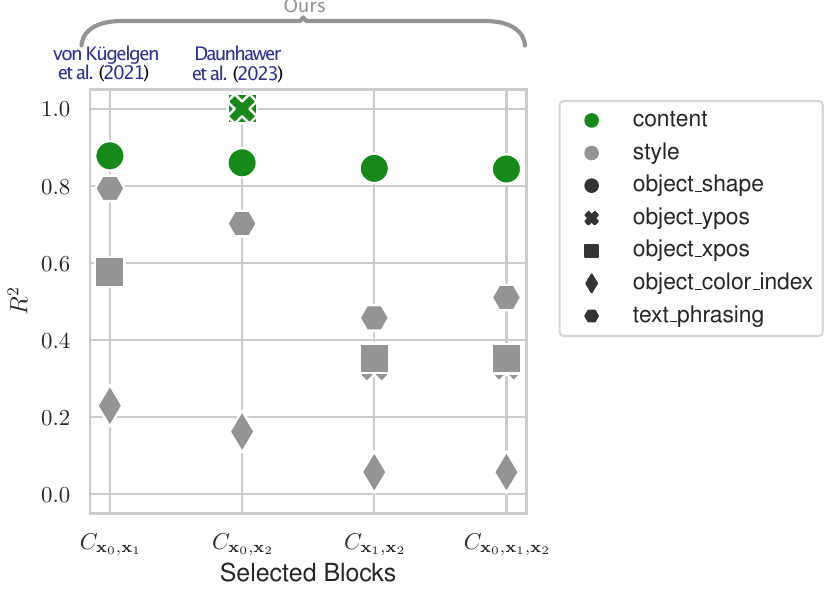}
\end{subfigure}
\caption{\small \textbf{Simultaneous Multi-Content Identification using View-Specific Encoders.} Experimental results on~\emph{Multimodal3DIdent}. \emph{Left}: Image latents (averaged between two image views) \emph{Right}: Text latents.}
\vspace{-15pt}
\label{fig:m3di_res}
\end{figure}

\subsection{Multi-Modal Content-Style Identifiability under Partial Observability}
\label{subsec:case2}
\looseness=-1\textbf{Experimental setup.} We experiment on a set of \emph{three} views $(\mathrm{img}_0, \mathrm{img}_1, \mathrm{txt}_0)$
extending both~\citep{daunhawer2023identifiability,von2021self}, which are limited to two views, either two images or one image and its caption. The second image view $\mathrm{img}_1$ is generated by perturbing a subset of latents of $\mathrm{img}_0$ as in~\citep{von2021self}. Notice that this setup provides perfect partial observability, because the text is generated using text-specific modality variables that are not involved in any image views e.g., \emph{text phrasing}. %
We train \emph{view-specific encoders} to learn \emph{all content blocks simultaneously} and predict individual latent variables from the each \emph{learned} content blocks. We assume access to the ground truth content indices to better match the baselines, but we relax this in~\cref{app:m3di}.

\looseness=-1\textbf{Discussion.}
~\Cref{fig:m3di_res} reports the $R^2$ on the ground truth latent values, predicted from the \emph{simultaneously learned multiple content blocks}~($C_{\xb_0, \xb_1}, C_{\xb_0, \xb_2}, C_{\xb_1, \xb_2}, C_{\xb_0, \xb_1, \xb_2}$, respectively). We remark that this \emph{single} experiment recovers both experimental setups from~\citep[][Sec 5.2]{von2021self},~\citep[][Sec 5.2]{daunhawer2023identifiability}: $C_{\xb_0, \xb_1}$ represents the content block from the image pairs $(\mathrm{img}_0, \mathrm{img}_1)$, which aligns with the setting in~\citep{von2021self} and $C_{\xb_0, \xb_2}$ shows the content block from the multi-modal pair $(\mathrm{img}_0, \mathrm{txt}_0)$, which is studied by~\citet{daunhawer2023identifiability}. We observe that the same performance for both prior works~\citep{von2021self,daunhawer2023identifiability} has been successfully reproduced from our single training process, which verifies the effectiveness and efficiency of~\cref{thm:general_ID_from_sets_size_unknown}.  Extended evaluation and more experimental details are provided in~\cref{app:m3di}. 
\subsection{Multi-Task Disentanglement}
\label{subsec:multi_task_disent}
\looseness=-1\textbf{Experimental setup} We follow~\cref{example:intuitive} with \emph{latent causal relations} to verify that: (i) the improved classification performance from~\citep{lachapelle2022synergies,fumero2023leveraging} originates from the fact that the task-related information is shared across multiple views (different observations from the same class) and (ii) this information can be identified~\pcref{thm:ID_from_sets}, \textit{even though the latent variables are not independent}. This explains the good performance of~\citep{fumero2023leveraging} on real-world data sets, where the latent variables are likely not independent, violating their theory. %

\looseness=-1\textbf{Discussion.} We synthetically generate the labels by linear/nonlinear labeling functions on the shared content values to resemble~\citep{lachapelle2022synergies,fumero2023leveraging}. As expected, the learned representation significantly eases the classification task and achieves an accuracy of 0.99 with linear and nonlinear labeling functions within 1k update steps, even with latent causal relations. This experimental result justifies that the success in the empirical evaluation of~\citep{fumero2023leveraging} can be explained by our theoretical framework, as discussed in~\cref{sec:related_work}. 

\section{Discussion and Conclusion}
\label{sec:discussion_conclusion}
\looseness=-1This paper revisits the problem of identifying possibly dependent latent variables under multiple partial non-linear measurements. Our theoretical results extend to an arbitrary number of views, each potentially measuring a strict subset of the latent variables. In our experiments, we validate our claims and demonstrate how prior work can be obtained as a special case of our setting. While our assumptions are relatively mild, we still have notable gaps between theory and practice, thoroughly discussed in~\cref{sec:extended_discussion}. In particular, we highlight discrete variables and finite-sample errors as common gaps, which we address only empirically. Interestingly, our work offers potential connections with work in the causality literature~\citep{triantafillou2010learning,gresele2022causal,mejia2022obtaining,guo2023out}. Discovering hidden causal structures from overlapping but not simultaneously observed marginals (e.g., via views collected in different experimental studies at different times) remains open for future works.

\subsubsection*{Reproducibility Statement} The datasets used in~\pcref{sec:experiment} are published by~\citet{gondal2019transfer,zimmermann2021contrastive,von2021self,daunhawer2023identifiability}. Results provided in the experiments section~\pcref{sec:experiment} can be reproduced using the implementation details provided in~\cref{sec:exp_details}. The code is available at \href{https://github.com/CausalLearningAI/multiview-crl}{https://github.com/CausalLearningAI/multiview-crl}. The part of implementation to replicate the experiments of~\citet{fumero2023leveraging} in~\cref{subsec:multi_task_disent} was kindly provided by the authors upon request, and we do not include it in the git repository.

\subsubsection*{Acknowledgements}
This work was initiated at the Second Bellairs Workshop on Causality held at the Bellairs Research Institute, January 6--13, 2022; we thank all workshop participants for providing a stimulating research environment.
Further, we thank Cian Eastwood, Luigi Gresele, Stefano Soatto, Marco Bagatella and A.\ René Geist for helpful discussion.
GM is a member of the Machine Learning Cluster of Excellence, EXC
number 2064/1 – Project number 390727645. JvK and GM acknowledge support from the German Federal
Ministry of Education and Research (BMBF) through the Tübingen AI Center (FKZ: 01IS18039B). %
The research of DX and SM was supported by the Air Force Office of Scientific Research under award number FA8655-22-1-7155. Any opinions, findings, and conclusions or recommendations expressed in this material are those of the author(s) and do not necessarily reflect the views of the United States Air Force.
We also thank SURF for the support in using the Dutch National Supercomputer Snellius. 
SL was supported by an IVADO
excellence PhD scholarship and by Samsung Electronics Co., Ldt.
DY was supported by an Amazon fellowship, the International Max Planck Research School for Intelligent Systems (IMPRS-IS) and the ISTA graduate school. 
Work done outside of Amazon.

\bibliography{iclr2024_conference}

\begin{thebibliography}{82}
\providecommand{\natexlab}[1]{#1}
\providecommand{\url}[1]{\texttt{#1}}
\expandafter\ifx\csname urlstyle\endcsname\relax
  \providecommand{\doi}[1]{doi: #1}\else
  \providecommand{\doi}{doi: \begingroup \urlstyle{rm}\Url}\fi

\bibitem[Adams et~al.(2021)Adams, Hansen, and Zhang]{adams2021identification}
Jeffrey Adams, Niels Hansen, and Kun Zhang.
\newblock Identification of partially observed linear causal models: Graphical conditions for the non-gaussian and heterogeneous cases.
\newblock \emph{Advances in Neural Information Processing Systems}, 34:\penalty0 22822--22833, 2021.

\bibitem[Ahuja et~al.(2022)Ahuja, Hartford, and Bengio]{ahuja2022weakly}
Kartik Ahuja, Jason~S Hartford, and Yoshua Bengio.
\newblock Weakly supervised representation learning with sparse perturbations.
\newblock \emph{Advances in Neural Information Processing Systems}, 35:\penalty0 15516--15528, 2022.

\bibitem[Ahuja et~al.(2023)Ahuja, Mahajan, Wang, and Bengio]{ahuja2023interventional}
Kartik Ahuja, Divyat Mahajan, Yixin Wang, and Yoshua Bengio.
\newblock Interventional causal representation learning.
\newblock In \emph{International Conference on Machine Learning}, pp.\  372--407. PMLR, 2023.

\bibitem[Amjad \& Geiger(2020)Amjad and Geiger]{amjad2020learning}
Rana~Ali Amjad and Bernhard~C. Geiger.
\newblock Learning representations for neural network-based classification using the information bottleneck principle.
\newblock \emph{{IEEE} Transactions on Pattern Analysis and Machine Intelligence}, 42\penalty0 (9):\penalty0 2225--2239, sep 2020.

\bibitem[Andrew et~al.(2013)Andrew, Arora, Bilmes, and Livescu]{andrew2013deep}
Galen Andrew, Raman Arora, Jeff Bilmes, and Karen Livescu.
\newblock Deep canonical correlation analysis.
\newblock In Sanjoy Dasgupta and David McAllester (eds.), \emph{Proceedings of the 30th International Conference on Machine Learning}, volume~28 of \emph{Proceedings of Machine Learning Research}, pp.\  1247--1255, Atlanta, Georgia, USA, 17--19 Jun 2013. PMLR.

\bibitem[Bengio et~al.(2013)Bengio, Courville, and Vincent]{bengio2013representation}
Yoshua Bengio, Aaron Courville, and Pascal Vincent.
\newblock Representation learning: A review and new perspectives.
\newblock \emph{IEEE transactions on pattern analysis and machine intelligence}, 35\penalty0 (8):\penalty0 1798--1828, 2013.

\bibitem[Benton et~al.(2017)Benton, Khayrallah, Gujral, Reisinger, Zhang, and Arora]{benton2017deep}
Adrian Benton, Huda Khayrallah, Biman Gujral, Dee~Ann Reisinger, Sheng Zhang, and Raman Arora.
\newblock Deep generalized canonical correlation analysis.
\newblock \emph{arXiv preprint arXiv:1702.02519}, 2017.

\bibitem[Brehmer et~al.(2022)Brehmer, De~Haan, Lippe, and Cohen]{brehmer2022weakly}
Johann Brehmer, Pim De~Haan, Phillip Lippe, and Taco~S Cohen.
\newblock Weakly supervised causal representation learning.
\newblock \emph{Advances in Neural Information Processing Systems}, 35:\penalty0 38319--38331, 2022.

\bibitem[Brown et~al.(2001)Brown, Yamada, and Sejnowski]{brown2001independent}
Glen~D Brown, Satoshi Yamada, and Terrence~J Sejnowski.
\newblock Independent component analysis at the neural cocktail party.
\newblock \emph{Trends in neurosciences}, 24\penalty0 (1):\penalty0 54--63, 2001.

\bibitem[Buchholz et~al.(2023)Buchholz, Rajendran, Rosenfeld, Aragam, Sch{\"o}lkopf, and Ravikumar]{buchholz2023learning}
Simon Buchholz, Goutham Rajendran, Elan Rosenfeld, Bryon Aragam, Bernhard Sch{\"o}lkopf, and Pradeep Ravikumar.
\newblock Learning linear causal representations from interventions under general nonlinear mixing.
\newblock \emph{arXiv preprint arXiv:2306.02235}, 2023.

\bibitem[Cai et~al.(2019)Cai, Xie, Glymour, Hao, and Zhang]{cai2019triad}
Ruichu Cai, Feng Xie, Clark Glymour, Zhifeng Hao, and Kun Zhang.
\newblock Triad constraints for learning causal structure of latent variables.
\newblock In \emph{Advances in Neural Information Processing Systems}, volume~32, pp.\  12883--12892, 2019.

\bibitem[Caron et~al.(2021)Caron, Touvron, Misra, J\'egou, Mairal, Bojanowski, and Joulin]{caron2021emerging}
Mathilde Caron, Hugo Touvron, Ishan Misra, Herv\'e J\'egou, Julien Mairal, Piotr Bojanowski, and Armand Joulin.
\newblock Emerging properties in self-supervised vision transformers.
\newblock In \emph{Proceedings of the International Conference on Computer Vision (ICCV)}, 2021.

\bibitem[Chadan \& Sabatier(2012)Chadan and Sabatier]{chadan2012inverse}
Khosrow Chadan and Pierre~C Sabatier.
\newblock \emph{Inverse problems in quantum scattering theory}.
\newblock Springer Science \& Business Media, 2012.

\bibitem[Chen et~al.(2020)Chen, Kornblith, Norouzi, and Hinton]{chen2020simple}
Ting Chen, Simon Kornblith, Mohammad Norouzi, and Geoffrey Hinton.
\newblock A simple framework for contrastive learning of visual representations.
\newblock In \emph{International conference on machine learning}, pp.\  1597--1607. PMLR, 2020.

\bibitem[Cherry(1953)]{cherry1953some}
E~Colin Cherry.
\newblock Some experiments on the recognition of speech, with one and with two ears.
\newblock \emph{The Journal of the acoustical society of America}, 25\penalty0 (5):\penalty0 975--979, 1953.

\bibitem[Comon(1994)]{comon1994independent}
Pierre Comon.
\newblock Independent component analysis, a new concept?
\newblock \emph{Signal processing}, 36\penalty0 (3):\penalty0 287--314, 1994.

\bibitem[Darmois(1951)]{darmois1951analyse}
George Darmois.
\newblock Analyse des liaisons de probabilit{\'e}.
\newblock In \emph{Proc. Int. Stat. Conferences 1947}, pp.\  231, 1951.

\bibitem[Daunhawer et~al.(2023)Daunhawer, Bizeul, Palumbo, Marx, and Vogt]{daunhawer2023identifiability}
Imant Daunhawer, Alice Bizeul, Emanuele Palumbo, Alexander Marx, and Julia~E Vogt.
\newblock Identifiability results for multimodal contrastive learning.
\newblock In \emph{The Eleventh International Conference on Learning Representations}, 2023.

\bibitem[Donoho(2006)]{donoho2006compressed}
David~L Donoho.
\newblock Compressed sensing.
\newblock \emph{IEEE Transactions on information theory}, 52\penalty0 (4):\penalty0 1289--1306, 2006.

\bibitem[Fumero et~al.(2023)Fumero, Wenzel, Zancato, Achille, Rodol{\'a}, Soatto, Sch{\"o}lkopf, and Locatello]{fumero2023leveraging}
Marco Fumero, Florian Wenzel, Luca Zancato, Alessandro Achille, Emanuele Rodol{\'a}, Stefano Soatto, Bernhard Sch{\"o}lkopf, and Francesco Locatello.
\newblock Leveraging sparse and shared feature activations for disentangled representation learning, 2023.

\bibitem[Gondal et~al.(2019)Gondal, Wuthrich, Miladinovic, Locatello, Breidt, Volchkov, Akpo, Bachem, Sch{\"o}lkopf, and Bauer]{gondal2019transfer}
Muhammad~Waleed Gondal, Manuel Wuthrich, Djordje Miladinovic, Francesco Locatello, Martin Breidt, Valentin Volchkov, Joel Akpo, Olivier Bachem, Bernhard Sch{\"o}lkopf, and Stefan Bauer.
\newblock On the transfer of inductive bias from simulation to the real world: a new disentanglement dataset.
\newblock \emph{Advances in Neural Information Processing Systems}, 32, 2019.

\bibitem[Gresele et~al.(2020)Gresele, Rubenstein, Mehrjou, Locatello, and Sch{\"o}lkopf]{gresele2019incomplete}
Luigi Gresele, Paul~K Rubenstein, Arash Mehrjou, Francesco Locatello, and Bernhard Sch{\"o}lkopf.
\newblock The incomplete rosetta stone problem: Identifiability results for multi-view nonlinear ica.
\newblock In \emph{Uncertainty in Artificial Intelligence}, pp.\  217--227. PMLR, 2020.

\bibitem[Gresele et~al.(2022)Gresele, von K{\"u}gelgen, K{\"u}bler, Kirschbaum, Sch{\"o}lkopf, and Janzing]{gresele2022causal}
Luigi Gresele, Julius von K{\"u}gelgen, Jonas K{\"u}bler, Elke Kirschbaum, Bernhard Sch{\"o}lkopf, and Dominik Janzing.
\newblock Causal inference through the structural causal marginal problem.
\newblock In \emph{International Conference on Machine Learning}, pp.\  7793--7824. PMLR, 2022.

\bibitem[Guo et~al.(2023)Guo, Wildberger, and Sch{\"o}lkopf]{guo2023out}
Siyuan Guo, Jonas Wildberger, and Bernhard Sch{\"o}lkopf.
\newblock Out-of-variable generalization.
\newblock \emph{arXiv preprint arXiv:2304.07896}, 2023.

\bibitem[Haykin(1994)]{haykin1994neural}
Simon Haykin.
\newblock \emph{Neural networks: a comprehensive foundation}.
\newblock Prentice Hall PTR, 1994.

\bibitem[He et~al.(2016)He, Zhang, Ren, and Sun]{he2015deep}
Kaiming He, Xiangyu Zhang, Shaoqing Ren, and Jian Sun.
\newblock Deep residual learning for image recognition.
\newblock In \emph{Proceedings of the IEEE conference on computer vision and pattern recognition}, pp.\  770--778, 2016.

\bibitem[Higgins et~al.(2017)Higgins, Matthey, Pal, Burgess, Glorot, Botvinick, Mohamed, and Lerchner]{higgins2017beta}
Irina Higgins, Loic Matthey, Arka Pal, Christopher Burgess, Xavier Glorot, Matthew Botvinick, Shakir Mohamed, and Alexander Lerchner.
\newblock beta-vae: Learning basic visual concepts with a constrained variational framework.
\newblock In \emph{International conference on learning representations}, 2017.

\bibitem[Higgins et~al.(2018)Higgins, Amos, Pfau, Racaniere, Matthey, Rezende, and Lerchner]{higgins2018definition}
Irina Higgins, David Amos, David Pfau, Sebastien Racaniere, Loic Matthey, Danilo Rezende, and Alexander Lerchner.
\newblock Towards a definition of disentangled representations.
\newblock \emph{arXiv preprint arXiv:1812.02230}, 2018.

\bibitem[Hyv{\"a}rinen \& Erkki(2000)Hyv{\"a}rinen and Erkki]{hyvarinen2000independent}
Aapo Hyv{\"a}rinen and Oja Erkki.
\newblock Independent component analysis: algorithms and applications.
\newblock \emph{Neural networks}, 13\penalty0 (4-5):\penalty0 411--430, 2000.

\bibitem[Hyv{\"a}rinen \& Pajunen(1999)Hyv{\"a}rinen and Pajunen]{hyvarinen1999nonlinear}
Aapo Hyv{\"a}rinen and Petteri Pajunen.
\newblock Nonlinear independent component analysis: Existence and uniqueness results.
\newblock \emph{Neural networks}, 12\penalty0 (3):\penalty0 429--439, 1999.

\bibitem[Hyvarinen et~al.(2019)Hyvarinen, Sasaki, and Turner]{hyvarinen2019nonlinear}
Aapo Hyvarinen, Hiroaki Sasaki, and Richard~E. Turner.
\newblock Nonlinear ica using auxiliary variables and generalized contrastive learning, 2019.

\bibitem[Jang et~al.(2016)Jang, Gu, and Poole]{jang2017categorical}
Eric Jang, Shixiang Gu, and Ben Poole.
\newblock Categorical reparameterization with gumbel-softmax.
\newblock \emph{arXiv preprint arXiv:1611.01144}, 2016.

\bibitem[Khemakhem et~al.(2020)Khemakhem, Monti, Kingma, and Hyvarinen]{khemakhem2020ice}
Ilyes Khemakhem, Ricardo Monti, Diederik Kingma, and Aapo Hyvarinen.
\newblock Ice-beem: Identifiable conditional energy-based deep models based on nonlinear ica.
\newblock volume~33, pp.\  12768--12778, 2020.

\bibitem[Kingma \& Ba(2014)Kingma and Ba]{kingma2017adam}
Diederik~P Kingma and Jimmy Ba.
\newblock Adam: A method for stochastic optimization.
\newblock \emph{arXiv preprint arXiv:1412.6980}, 2014.

\bibitem[Kingma \& Welling(2013)Kingma and Welling]{kingma2022autoencoding}
Diederik~P Kingma and Max Welling.
\newblock Auto-encoding variational bayes.
\newblock \emph{arXiv preprint arXiv:1312.6114}, 2013.

\bibitem[Kivva et~al.(2021)Kivva, Rajendran, Ravikumar, and Aragam]{kivva2021learning}
Bohdan Kivva, Goutham Rajendran, Pradeep Ravikumar, and Bryon Aragam.
\newblock Learning latent causal graphs via mixture oracles.
\newblock In \emph{Advances in Neural Information Processing Systems}, volume~34, pp.\  18087--18101, 2021.

\bibitem[Kivva et~al.(2022)Kivva, Rajendran, Ravikumar, and Aragam]{kivva2022identifiability}
Bohdan Kivva, Goutham Rajendran, Pradeep Ravikumar, and Bryon Aragam.
\newblock Identifiability of deep generative models without auxiliary information.
\newblock In S.~Koyejo, S.~Mohamed, A.~Agarwal, D.~Belgrave, K.~Cho, and A.~Oh (eds.), \emph{Advances in Neural Information Processing Systems}, volume~35, pp.\  15687--15701. {Curran Associates, Inc.}, 2022.

\bibitem[Klindt et~al.(2021)Klindt, Schott, Sharma, Ustyuzhaninov, Brendel, Bethge, and Paiton]{klindt2021towards}
David~A Klindt, Lukas Schott, Yash Sharma, Ivan Ustyuzhaninov, Wieland Brendel, Matthias Bethge, and Dylan Paiton.
\newblock Towards nonlinear disentanglement in natural data with temporal sparse coding.
\newblock In \emph{International Conference on Learning Representations}, 2021.

\bibitem[Kong et~al.(2022)Kong, Xie, Yao, Zheng, Chen, Stojanov, Akinwande, and Zhang]{kong2022partial}
Lingjing Kong, Shaoan Xie, Weiran Yao, Yujia Zheng, Guangyi Chen, Petar Stojanov, Victor Akinwande, and Kun Zhang.
\newblock Partial disentanglement for domain adaptation.
\newblock In \emph{International Conference on Machine Learning}, pp.\  11455--11472. PMLR, 2022.

\bibitem[Lachapelle et~al.(2022)Lachapelle, Pau, Sharma, Everett, {Le Priol}, Lacoste, and Lacoste-Julien]{lachapelle2022disentanglement}
S{\'e}bastien Lachapelle, {Rodriguez Lopez,} Pau, Yash Sharma, Katie~E. Everett, R{\'e}mi {Le Priol}, Alexandre Lacoste, and Simon Lacoste-Julien.
\newblock Disentanglement via mechanism sparsity regularization: A new principle for nonlinear {ICA}.
\newblock In \emph{First Conference on Causal Learning and Reasoning}, 2022.

\bibitem[Lachapelle et~al.(2023)Lachapelle, Deleu, Mahajan, Mitliagkas, Bengio, Lacoste-Julien, and Bertrand]{lachapelle2022synergies}
S{\'e}bastien Lachapelle, Tristan Deleu, Divyat Mahajan, Ioannis Mitliagkas, Yoshua Bengio, Simon Lacoste-Julien, and Quentin Bertrand.
\newblock Synergies between disentanglement and sparsity: Generalization and identifiability in multi-task learning.
\newblock In \emph{International Conference on Machine Learning}, pp.\  18171--18206. PMLR, 2023.

\bibitem[Liang et~al.(2023)Liang, Keki{\'c}, von K{\"u}gelgen, Buchholz, Besserve, Gresele, and Sch{\"o}lkopf]{liang2023causal}
Wendong Liang, Armin Keki{\'c}, Julius von K{\"u}gelgen, Simon Buchholz, Michel Besserve, Luigi Gresele, and Bernhard Sch{\"o}lkopf.
\newblock Causal component analysis.
\newblock \emph{arXiv preprint arXiv:2305.17225}, 2023.

\bibitem[Lippe et~al.(2022)Lippe, Magliacane, L{\"o}we, Asano, Cohen, and Gavves]{lippe2022citris}
Phillip Lippe, Sara Magliacane, Sindy L{\"o}we, Yuki~M Asano, Taco Cohen, and Stratis Gavves.
\newblock Citris: Causal identifiability from temporal intervened sequences.
\newblock In \emph{International Conference on Machine Learning}, pp.\  13557--13603. PMLR, 2022.

\bibitem[Liu et~al.(2022)Liu, Zhang, Gong, Gong, Huang, Hengel, Zhang, and Shi]{liu2022identifying}
Yuhang Liu, Zhen Zhang, Dong Gong, Mingming Gong, Biwei Huang, Anton van~den Hengel, Kun Zhang, and Javen~Qinfeng Shi.
\newblock Identifying weight-variant latent causal models.
\newblock \emph{arXiv preprint arXiv:2208.14153}, 2022.

\bibitem[Locatello et~al.(2019)Locatello, Bauer, Lucic, Raetsch, Gelly, Sch{\"o}lkopf, and Bachem]{locatello2019challenging}
Francesco Locatello, Stefan Bauer, Mario Lucic, Gunnar Raetsch, Sylvain Gelly, Bernhard Sch{\"o}lkopf, and Olivier Bachem.
\newblock Challenging common assumptions in the unsupervised learning of disentangled representations.
\newblock In \emph{international conference on machine learning}, pp.\  4114--4124. PMLR, 2019.

\bibitem[Locatello et~al.(2020)Locatello, Poole, Raetsch, Sch{\"o}lkopf, Bachem, and Tschannen]{locatello2020weakly}
Francesco Locatello, Ben Poole, Gunnar Raetsch, Bernhard Sch{\"o}lkopf, Olivier Bachem, and Michael Tschannen.
\newblock Weakly-supervised disentanglement without compromises.
\newblock In Hal~Daumé III and Aarti Singh (eds.), \emph{Proceedings of the 37th International Conference on Machine Learning}, volume 119 of \emph{Proceedings of Machine Learning Research}, pp.\  6348--6359. PMLR, 13--18 Jul 2020.

\bibitem[Lyu \& Fu(2020)Lyu and Fu]{lyu2020nonlinear}
Qi~Lyu and Xiao Fu.
\newblock Nonlinear multiview analysis: Identifiability and neural network-assisted implementation.
\newblock \emph{IEEE Transactions on Signal Processing}, 68:\penalty0 2697--2712, 2020.

\bibitem[Lyu et~al.(2021)Lyu, Fu, Wang, and Lu]{lyu2021understanding}
Qi~Lyu, Xiao Fu, Weiran Wang, and Songtao Lu.
\newblock Understanding latent correlation-based multiview learning and self-supervision: An identifiability perspective.
\newblock \emph{arXiv preprint arXiv:2106.07115}, 2021.

\bibitem[Mejia et~al.(2022)Mejia, Kirschbaum, and Janzing]{mejia2022obtaining}
Sergio H~Garrido Mejia, Elke Kirschbaum, and Dominik Janzing.
\newblock Obtaining causal information by merging datasets with maxent.
\newblock In \emph{International Conference on Artificial Intelligence and Statistics}, pp.\  581--603. PMLR, 2022.

\bibitem[Michlo(2021)]{michlo2021Disent}
Nathan~Juraj Michlo.
\newblock Disent - a modular disentangled representation learning framework for pytorch.
\newblock Github, 2021.

\bibitem[Morioka \& Hyv{\"a}rinen(2023)Morioka and Hyv{\"a}rinen]{morioka2023causal}
Hiroshi Morioka and Aapo Hyv{\"a}rinen.
\newblock Causal representation learning made identifiable by grouping of observational variables.
\newblock \emph{arXiv preprint arXiv:2310.15709}, 2023.

\bibitem[Morioka \& Hyvarinen(2023)Morioka and Hyvarinen]{morioka2023connectivity}
Hiroshi Morioka and Aapo Hyvarinen.
\newblock Connectivity-contrastive learning: Combining causal discovery and representation learning for multimodal data.
\newblock In Francisco Ruiz, Jennifer Dy, and Jan-Willem van~de Meent (eds.), \emph{Proceedings of The 26th International Conference on Artificial Intelligence and Statistics}, volume 206 of \emph{Proceedings of Machine Learning Research}, pp.\  3399--3426. PMLR, 25--27 Apr 2023.

\bibitem[Oord et~al.(2018)Oord, Li, and Vinyals]{oord2019representation}
Aaron van~den Oord, Yazhe Li, and Oriol Vinyals.
\newblock Representation learning with contrastive predictive coding.
\newblock \emph{arXiv preprint arXiv:1807.03748}, 2018.

\bibitem[Oquab et~al.(2023)Oquab, Darcet, Moutakanni, Vo, Szafraniec, Khalidov, Fernandez, Haziza, Massa, El-Nouby, et~al.]{oquab2023dinov2}
Maxime Oquab, Timoth{\'e}e Darcet, Th{\'e}o Moutakanni, Huy Vo, Marc Szafraniec, Vasil Khalidov, Pierre Fernandez, Daniel Haziza, Francisco Massa, Alaaeldin El-Nouby, et~al.
\newblock Dinov2: Learning robust visual features without supervision.
\newblock \emph{arXiv preprint arXiv:2304.07193}, 2023.

\bibitem[Pandeva \& Forr{\'e}(2023)Pandeva and Forr{\'e}]{pandeva2023multi}
Teodora Pandeva and Patrick Forr{\'e}.
\newblock Multi-view independent component analysis with shared and individual sources.
\newblock In \emph{Uncertainty in Artificial Intelligence}, pp.\  1639--1650. PMLR, 2023.

\bibitem[Ristaniemi(1999)]{ristaniemi1999performance}
Tapani Ristaniemi.
\newblock On the performance of blind source separation in cdma downlink.
\newblock In \emph{Proceedings of the International Workshop on Independent Component Analysis and Signal Separation (ICA'99)}, pp.\  437--441, 1999.

\bibitem[Sch{\"o}lkopf et~al.(2016)Sch{\"o}lkopf, Hogg, Wang, Foreman-Mackey, Janzing, Simon-Gabriel, and Peters]{scholkopf2016modeling}
Bernhard Sch{\"o}lkopf, David~W Hogg, Dun Wang, Daniel Foreman-Mackey, Dominik Janzing, Carl-Johann Simon-Gabriel, and Jonas Peters.
\newblock Modeling confounding by half-sibling regression.
\newblock \emph{Proceedings of the National Academy of Sciences}, 113\penalty0 (27):\penalty0 7391--7398, 2016.

\bibitem[Sch{\"o}lkopf et~al.(2021)Sch{\"o}lkopf, Locatello, Bauer, Ke, Kalchbrenner, Goyal, and Bengio]{scholkopf2021toward}
Bernhard Sch{\"o}lkopf, Francesco Locatello, Stefan Bauer, Nan~Rosemary Ke, Nal Kalchbrenner, Anirudh Goyal, and Yoshua Bengio.
\newblock Toward causal representation learning.
\newblock \emph{Proceedings of the IEEE}, 109\penalty0 (5):\penalty0 612--634, 2021.

\bibitem[Seitzer et~al.(2022)Seitzer, Horn, Zadaianchuk, Zietlow, Xiao, Simon-Gabriel, He, Zhang, Sch{\"o}lkopf, Brox, et~al.]{seitzer2023bridging}
Maximilian Seitzer, Max Horn, Andrii Zadaianchuk, Dominik Zietlow, Tianjun Xiao, Carl-Johann Simon-Gabriel, Tong He, Zheng Zhang, Bernhard Sch{\"o}lkopf, Thomas Brox, et~al.
\newblock Bridging the gap to real-world object-centric learning.
\newblock In \emph{The Eleventh International Conference on Learning Representations}, 2022.

\bibitem[Silva et~al.(2006)Silva, Scheines, Glymour, Spirtes, and Chickering]{silva2006learning}
Ricardo Silva, Richard Scheines, Clark Glymour, Peter Spirtes, and David~Maxwell Chickering.
\newblock Learning the structure of linear latent variable models.
\newblock \emph{Journal of Machine Learning Research}, 7\penalty0 (2), 2006.

\bibitem[Squires et~al.(2023)Squires, Seigal, Bhate, and Uhler]{squires2023linear}
Chandler Squires, Anna Seigal, Salil~S. Bhate, and Caroline Uhler.
\newblock Linear causal disentanglement via interventions.
\newblock In \emph{International Conference on Machine Learning}, volume 202, pp.\  32540--32560. {PMLR}, 2023.

\bibitem[Sturma et~al.(2023)Sturma, Squires, Drton, and Uhler]{sturma2023unpaired}
Nils Sturma, Chandler Squires, Mathias Drton, and Caroline Uhler.
\newblock Unpaired multi-domain causal representation learning, 2023.

\bibitem[Tian et~al.(2020)Tian, Sun, Poole, Krishnan, Schmid, and Isola]{tian2020makes}
Yonglong Tian, Chen Sun, Ben Poole, Dilip Krishnan, Cordelia Schmid, and Phillip Isola.
\newblock What makes for good views for contrastive learning?
\newblock \emph{Advances in neural information processing systems}, 33:\penalty0 6827--6839, 2020.

\bibitem[Tosh et~al.(2021)Tosh, Krishnamurthy, and Hsu]{tosh2021contrastive}
Christopher Tosh, Akshay Krishnamurthy, and Daniel Hsu.
\newblock Contrastive learning, multi-view redundancy, and linear models.
\newblock In \emph{Algorithmic Learning Theory}, pp.\  1179--1206. PMLR, 2021.

\bibitem[Trapnell et~al.(2014)Trapnell, Cacchiarelli, Grimsby, Pokharel, Li, Morse, Lennon, Livak, Mikkelsen, and Rinn]{trapnell2014dynamics}
Cole Trapnell, Davide Cacchiarelli, Jonna Grimsby, Prapti Pokharel, Shuqiang Li, Michael Morse, Niall~J Lennon, Kenneth~J Livak, Tarjei~S Mikkelsen, and John~L Rinn.
\newblock The dynamics and regulators of cell fate decisions are revealed by pseudotemporal ordering of single cells.
\newblock \emph{Nature biotechnology}, 32\penalty0 (4):\penalty0 381--386, 2014.

\bibitem[Triantafillou et~al.(2010)Triantafillou, Tsamardinos, and Tollis]{triantafillou2010learning}
Sofia Triantafillou, Ioannis Tsamardinos, and Ioannis Tollis.
\newblock Learning causal structure from overlapping variable sets.
\newblock In \emph{Proceedings of the Thirteenth International Conference on Artificial Intelligence and Statistics}, pp.\  860--867, 2010.

\bibitem[Tsai et~al.(2020)Tsai, Wu, Salakhutdinov, and Morency]{tsai2020self}
Yao-Hung~Hubert Tsai, Yue Wu, Ruslan Salakhutdinov, and Louis-Philippe Morency.
\newblock Self-supervised learning from a multi-view perspective.
\newblock In \emph{International Conference on Learning Representations}, 2020.

\bibitem[Varici et~al.(2023)Varici, Acarturk, Shanmugam, Kumar, and Tajer]{varici2023score}
Burak Varici, Emre Acarturk, Karthikeyan Shanmugam, Abhishek Kumar, and Ali Tajer.
\newblock Score-based causal representation learning with interventions.
\newblock \emph{arXiv preprint arXiv:2301.08230}, 2023.

\bibitem[Vig{\'a}rio et~al.(1997)Vig{\'a}rio, Jousm{\"a}ki, H{\"a}m{\"a}l{\"a}inen, Hari, and Oja]{vigario1997independent}
Ricardo Vig{\'a}rio, Veikko Jousm{\"a}ki, Matti H{\"a}m{\"a}l{\"a}inen, Riitta Hari, and Erkki Oja.
\newblock Independent component analysis for identification of artifacts in magnetoencephalographic recordings.
\newblock \emph{Advances in neural information processing systems}, 10, 1997.

\bibitem[von K{\"u}gelgen et~al.(2021)von K{\"u}gelgen, Sharma, Gresele, Brendel, Sch{\"o}lkopf, Besserve, and Locatello]{von2021self}
Julius von K{\"u}gelgen, Yash Sharma, Luigi Gresele, Wieland Brendel, Bernhard Sch{\"o}lkopf, Michel Besserve, and Francesco Locatello.
\newblock Self-supervised learning with data augmentations provably isolates content from style.
\newblock \emph{Advances in neural information processing systems}, 34:\penalty0 16451--16467, 2021.

\bibitem[von K{\"u}gelgen et~al.(2023)von K{\"u}gelgen, Besserve, Liang, Gresele, Keki{\'c}, Bareinboim, Blei, and Sch{\"o}lkopf]{von2023nonparametric}
Julius von K{\"u}gelgen, Michel Besserve, Wendong Liang, Luigi Gresele, Armin Keki{\'c}, Elias Bareinboim, David~M Blei, and Bernhard Sch{\"o}lkopf.
\newblock Nonparametric identifiability of causal representations from unknown interventions.
\newblock \emph{arXiv preprint arXiv:2306.00542}, 2023.

\bibitem[Willetts \& Paige(2021)Willetts and Paige]{willetts2021don}
Matthew Willetts and Brooks Paige.
\newblock I don't need u: Identifiable non-linear ica without side information.
\newblock \emph{arXiv preprint arXiv:2106.05238}, 2021.

\bibitem[Wunsch(1996)]{wunsch1996ocean}
Carl Wunsch.
\newblock \emph{The ocean circulation inverse problem}.
\newblock Cambridge University Press, 1996.

\bibitem[Xie et~al.(2020)Xie, Cai, Huang, Glymour, Hao, and Zhang]{xie2020generalized}
Feng Xie, Ruichu Cai, Biwei Huang, Clark Glymour, Zhifeng Hao, and Kun Zhang.
\newblock Generalized independent noise condition for estimating latent variable causal graphs.
\newblock In \emph{Advances in Neural Information Processing Systems}, volume~33, pp.\  14891--14902, 2020.

\bibitem[Xie et~al.(2022)Xie, Huang, Chen, He, Geng, and Zhang]{xie2022identification}
Feng Xie, Biwei Huang, Zhengming Chen, Yangbo He, Zhi Geng, and Kun Zhang.
\newblock Identification of linear non-gaussian latent hierarchical structure.
\newblock In \emph{International Conference on Machine Learning}, pp.\  24370--24387. PMLR, 2022.

\bibitem[Xu et~al.(2015)Xu, Wang, Chen, and Li]{xu2015empirical}
Bing Xu, Naiyan Wang, Tianqi Chen, and Mu~Li.
\newblock Empirical evaluation of rectified activations in convolutional network.
\newblock \emph{arXiv preprint arXiv:1505.00853}, 2015.

\bibitem[Zbontar et~al.(2021)Zbontar, Jing, Misra, LeCun, and Deny]{zbontar2021barlow}
Jure Zbontar, Li~Jing, Ishan Misra, Yann LeCun, and St{\'e}phane Deny.
\newblock Barlow twins: Self-supervised learning via redundancy reduction.
\newblock \emph{arXiv preprint arXiv:2103.03230}, 2021.

\bibitem[Zhang et~al.(2023)Zhang, Squires, Greenewald, Srivastava, Shanmugam, and Uhler]{zhang2023identifiability}
Jiaqi Zhang, Chandler Squires, Kristjan Greenewald, Akash Srivastava, Karthikeyan Shanmugam, and Caroline Uhler.
\newblock Identifiability guarantees for causal disentanglement from soft interventions.
\newblock \emph{arXiv preprint arXiv:2307.06250}, 2023.

\bibitem[Zhang \& Hyv{\"a}rinen(2009)Zhang and Hyv{\"a}rinen]{zhang2009identifiability}
K~Zhang and A~Hyv{\"a}rinen.
\newblock On the identifiability of the post-nonlinear causal model.
\newblock In \emph{25th Conference on Uncertainty in Artificial Intelligence (UAI 2009)}, pp.\  647--655. AUAI Press, 2009.

\bibitem[Zhang \& Chan(2006)Zhang and Chan]{zhang2006extensions}
Kun Zhang and Lai-Wan Chan.
\newblock Extensions of ica for causality discovery in the hong kong stock market.
\newblock In \emph{International Conference on Neural Information Processing}, pp.\  400--409. Springer, 2006.

\bibitem[Zheng et~al.(2022)Zheng, Ng, and Zhang]{zheng2022on}
Yujia Zheng, Ignavier Ng, and Kun Zhang.
\newblock On the identifiability of nonlinear {ICA}: Sparsity and beyond.
\newblock In Alice~H. Oh, Alekh Agarwal, Danielle Belgrave, and Kyunghyun Cho (eds.), \emph{Advances in Neural Information Processing Systems}, 2022.

\bibitem[Zimmermann et~al.(2021)Zimmermann, Sharma, Schneider, Bethge, and Brendel]{zimmermann2021contrastive}
Roland~S. Zimmermann, Yash Sharma, Steffen Schneider, Matthias Bethge, and Wieland Brendel.
\newblock Contrastive learning inverts the data generating process.
\newblock \emph{Proceedings of Machine Learning Research}, 139:\penalty0 12979--12990, 2021.

\end{thebibliography}
\bibliographystyle{iclr2024_conference} 
\clearpage
\appendix

\newcommand{\GreenLPlus}{\mathbin{\textcolor{Green}{[L+1]}}}
\newcommand{\OrangeL}{\mathbin{\textcolor{Orange}{V}}}
\newcommand{\gtilde}{\mathbin{\textcolor{Green}{\tilde{g}}}}
\newcommand{\g}{\mathbin{\textcolor{Orange}{g}}}
\newcommand{\htilde}{\mathbin{\textcolor{Green}{\tilde{h}}}}
\newcommand{\h}{\mathbin{\textcolor{Orange}{h}}}

\addcontentsline{toc}{section}{Appendix} %
\part{Appendix} %
{
  \hypersetup{linkcolor=Blue}
  \parttoc
}

\makenomenclature
\renewcommand{\nomname}{} %
\section{Notation and Terminology}
\label{sec:notations}
\vspace{-2em}
\nomenclature[1]{\(N\)}{Number of latents}
\nomenclature[2]{\(K\)}{Number of views}
\nomenclature[3]{\(j\)}{Index for latent variables}
\nomenclature[4]{\(k\)}{Index for views}
\nomenclature[5]{\(\Zcal\)}{Latent space}
\nomenclature[6]{\(\Xcal_k\)}{Observational space for $k$-th view}
\nomenclature[7]{\(\xb_k\)}{Observed $k$-th view}
\nomenclature[8]{\(S_k\)}{Index set for $k$-th view-specific latents}
\nomenclature[9]{\(V\)}{$\{1, \dots, l\}$}
\nomenclature[10]{\(C\)}{Index set for shared content variables}
\nomenclature[11]{\(\zb_C\)}{Shared content variables}
\nomenclature[12]{\(\Vcal\)}{Collection of subset of views from $V$}
\nomenclature[13]{\(V_i\)}{Index set for subset of views in set of views $V$}
\nomenclature[14]{\(C_{i}\)}{Index set for shared content variables from subset of views $V_i \in \Vcal$}
\nomenclature[15]{\(\zb_{C_i}\)}{Shared content variables from subset of views $V_i$}
\printnomenclature
\section{Related work and special cases of our theory}
\label{sec:extended_related_work}
We present our identifiability results from~\cref{sec:identifiability} as a unified framework implying several prior works in multi-view nonlinear ICA, disentanglement, and causal representation learning.%

\paragraph{Multi-View Nonlinear ICA} \citet{gresele2019incomplete} extend the idea of nonlinear ICA introduced by~\citet[][Sec. 3]{hyvarinen2019nonlinear} by allowing a more flexible relationship between the latents and the auxiliary variables: instead of imposing \emph{conditional independence} the shared source information $\mathbf{c}$ on some auxiliary observed variables, \citet{gresele2019incomplete} \emph{associate} the shared information source with some view-specific noise variable $n_i$ through some smooth mapping $g_k$:
    \begin{equation*}
        \xb_k = f_k(g_k(\mathbf{c}, \mathbf{n}_k)), \qquad k \in [K].
    \end{equation*}
Define the composition of the view-specific function $f_k$ and the noise corruption function $g_k$ as a new mixing function $\tilde{f}_k:= f_k \circ g_k.$, then each view $\xb_k, k \in [K]$ is generated by a \emph{view-specific mixing function} $\tilde{f}_k$ which takes the shared content $\mathbf{c}$ and some additional unobserved noise variable $n_k$ as input, that is: $\xb_k = \tilde{f}_k(\mathbf{c}, \mathbf{n}_k)$. In this case, the shared source information $\mathbf{s}$ together with the view-specific latent noise $n_k$ defines the view-specific latents $S_k$ in our notation. The shared source $\mathrm{c}$ corresponds to our content variables. Our results~\cref{thm:ID_from_sets} can be considered as a generalized version of~\citet[][Theorem~8]{gresele2019incomplete} for multiple views, by removing the additivity constraint of the corruption function $g$~\citep[][Sec 3.3]{gresele2019incomplete}.  

\citet{willetts2021don, kivva2022identifiability,liu2022identifying} show the nonlinear ICA problem can be solved using non-observable, learnable, clustering task variables $u$ to replace the observed the auxiliary variable in the traditional nonlinear ICA literature~\citep{hyvarinen1999nonlinear}. Conditioning on the \emph{same} latents on various clustering task enforces recovering the true latent factors (up to a bijective mapping). The idea of utilizing the clustering task goes hand in hand with contrastive self-supervised learning. Clustering itself can be considered as a soft relaxation of our hard global invariance condition in~\cref{eq:main_loss_set_size_unknown}, in the sense they enforce the shared, task-relevant features to be \emph{similar} within a cluster but not necessarily having the exact same value.

\paragraph{Weakly-Supervised Representation Learning}  \citet{locatello2020weakly} consider a pair of views (e.g., images) $(\xb_1, \xb_2)$ where ${\xb}_2$ is obtained by perturbing a subset of the generating factors of $\xb_1$. Formally,
    \begin{equation}
    \begin{aligned}
        \xb_1 = f(\zb)\qquad \xb_2 = f(\tilde{\zb}) \qquad\zb, \tilde{\zb} \in \RR^d,
    \end{aligned}
    \end{equation}
where $\zb_S = \tilde{\zb}_S$ while $\zb_{\bar S} \neq \tilde{\zb}_{\bar S}$ for some subset of latents $S \subseteq [d]$.
In this case, $\zb_S$ is the shared content between the pair of views $\xb_1, \xb_2$. According to the adaptive algorithm~\citep[Sec 4.1][]{locatello2020weakly}, the shared content is computed by averaging the encoded representation from $\xb_1, \xb_2$ across the shared dimensions, that is:
\begin{equation}
\begin{aligned}
    g(\xb_k)_j &\gets a(g(\xb_1)_j, g(\xb_2)_j) \qquad j \in S
\end{aligned}
\end{equation}
By substituting the extracted representation using the averaged value,~\citet{locatello2020weakly} achieve the same invariance condition as enforced in the first term of~\cref{eq:main_loss_set}.
The amortized encoder $g$ is trained to maximize the ELBO~\citep{kingma2022autoencoding} over the pair of views, which is equivalent to minimizing the reconstruction loss
\begin{equation}
    \EE \left[\xb - dec(g(\xb_1))\right] + \EE \left[\xb_2 - dec(g(\xb_2))\right].
\end{equation}
The reconstruction loss is minimized when the compression is lossless, or equivalently, the learned representation is uniformly distributed~\citep{zimmermann2021contrastive}. The uniformity of the representation, i.e., the lossless compression, can be enforced by maximizing the entropy term, as defined in~\cref{eq:main_loss_set}.
Theoretically, ~\citet[][Theorem 1.]{locatello2020weakly} have shown that the shared content $\zb_S$ can be recovered up to permutation, which aligns with our results~\cref{thm:general_ID_from_sets_size_unknown} that the shared content can be inferred up to a smooth invertible mapping. 

\citet{ahuja2022weakly} extend \citet{locatello2020weakly} by exploring more perturbations options to achieve full~\emph{disentanglement}. The main results~\citep[][Theorem 8]{ahuja2022weakly} state that each individual factor of a $d-$dimensional latents can be recovered (up to a bijection) when we augment the original observation with $d$ views, each obtained perturbing one unique latent component. This can be explained by~\cref{thm:general_ID_from_sets_size_unknown} because any $(d-1)$ views from this set would share exactly one latent component, which makes it identifiable. Although the theoretical claim by~\citet{ahuja2022weakly} is to some extent aligned with our theory, in practice, they explicitly require knowledge of the ground truth content indices while we do not necessarily. 

\paragraph{Mutual Information-based Framework} \citet{tian2020makes,tsai2020self} argue that the self-supervised signal should be approximately redundant to the task-related information. The self-supervised learning methods are based on extracting the task-relevant information (by maximizing the mutual information between the extracted representation $\hat{\zb}_{\xb}$ of the input $\xb$ and the self-supervised signal $\mathbf{s}$: $I(\hat{\zb}_{\xb}, \mathbf{s}))$ and discarding the task-irrelevant information conditioned on task $T$: $I(\xb, \mathbf{s}~|~T)$. The mutual information $I(\hat{\zb}_{\xb}, \mathbf{s}))$ is maximized if $\mathbf{s}$ is a deterministic function (for example, a MLP) of $\hat{\zb}$~\citep[][Theorem 1.]{amjad2020learning}. Since the mutual information remains invariant under deterministic transformation of the random variables, we have:
\begin{equation}
    \max I(\hat{\zb}, s) = \max I(\hat{\zb}, g(s)) = \max I(\hat{\zb}, \hat{\zb}) = \max H(\hat{\zb})
\end{equation}
which is equivalent to maximizing the entropy of the learned representations, as given in~\cref{eq:main_loss_set}.
Coupled with the empirically shown strong connection between the task-related information and shared content across multiple views~\citep{tian2020makes,tsai2020self,tosh2021contrastive, lachapelle2022synergies,fumero2023leveraging}, our results ~\pcref{thm:ID_from_sets,thm:general_ID_from_sets_size_unknown} provides a theoretical explanation for these approaches. As the shared content between the original view and self-supervised signal is proven to be related to the ground truth task-related information through a smooth invertible function, it is reasonable to see the usefulness of this high quality representation in downstream tasks.

\paragraph{Latent Correlation Maximization} Similar alignment conditions, as given in~\cref{eq:loss_set}, have been widely studied in the latent correlation maximization / latent component matching literature~\citep{andrew2013deep,benton2017deep,lyu2020nonlinear,lyu2021understanding}. \citet[][Theorem 1.]{lyu2021understanding} show that, by imposing additional invertibility constraint on the encoders latent correlation maximization across two views leads to identification of the shared component, up to an invertible function. This theoretical result can be considered as a explicit special case of~\cref{thm:ID_from_sets}, where we extend the identifiability proof to more than two multi-modal views.

\paragraph{Content-Style Identification} Our work is most closely related to~\citep{von2021self,daunhawer2023identifiability}, while our results extended prior work that purely focused on identifiability from pair of views~\citep{von2021self,daunhawer2023identifiability}. ~\citet[][Theorem 4.2]{von2021self} presented a special case of~\cref{thm:ID_from_sets} where the set size $l = 2$ and the mixing function $f_1 = f_2$ for both views; ~\citet[][Theorem 1]{daunhawer2023identifiability} formulate another special case of~\cref{thm:ID_from_sets} by allowing multi modality in the pair of views, but coming with the restriction that the view-specific modality variables have to be independent from others. From the data generating perspective, our work differs from prior work in the sense that all of the entangled views are simultaneously generated, each based on view-specific set of latent, while prior work generate the augmented (second) view by perturbing some style variables. In our case, ``style" is relative to specific views. Style variables could become the content block for some set of views~\pcref{thm:ID_from_sets} and thus be identifiable or can be inferred as independent complement block of the content~\pcref{cor:id_alg_complement}.~\citet{kong2022partial} have proven identifiability for \emph{independent} partitions in the latent space but mostly focus on the domain adaptation tasks where additional targets are required as supervision signals.

\paragraph{Multi-Task Disentanglement}
\citet{lachapelle2022synergies,fumero2023leveraging,zheng2022on} differs from our theory in the sense that their sparse classifier head jointly enforces the lossless compression (which we do with the entropy regularization) and a soft alignment up to a linear transformation (relaxing our hard alignment). In their setting, the different views are images of the same class and their augmentations sampled from a given task and the selector variable is implemented with the linear classifier. The identifiability principles we use of lossless compression, alignment, and information-sharing are similar. With this, we can explain the result that task-related and task-irrelevant information can be disentangled as blocks, as given in~\citet[][Theorem 3.1]{lachapelle2022synergies}, ~\citet[][Proposition 1.]{fumero2023leveraging}. With our theory, their identifiability results extend to non-independent blocks, which is an important case that is not covered in the original works.

\section{Proofs}
\label{sec:proofs}
\subsection{Proof for~\cref{thm:ID_from_sets}}
Our proof follows the steps from~\citet{von2021self} with slight adaptation:
\begin{enumerate}
    \item We show in~\cref{lemma:exist_opt_enc} that the lower bound of the loss~\cref{eq:main_loss_set} is zero and construct encoders $\{g^*_k: \Xcal_k \to (0, 1)^{|C|}\}_{k \in V}$ that reach this lower bound;
    \item Next, we show in~\cref{lemma:content_style_isolation_set_of_views} that for any set of encoders $\{g_k\}_{k \in V}$ that minimizes the loss, each learned $g_k(\xb_k)$ depends only on the shared content variables $\zb_C$, i.e. $g_k(\xb_k) = h_k(\zb_C)$ for some smooth function $h_k: \Zcal_C \to (0, 1)^{|C|}$.
    \item We conclude the proof by showing that every $h_k$ is invertible using~\cref{prop:zimmermann_invertibility_uniformality}~\citep[][Proposition 5.]{zimmermann2021contrastive}.
\end{enumerate}
We rephrase each step as a separate lemma and use them to complete the final proof for~\cref{thm:ID_from_sets}.

\begin{lemma}[Existence of Optimal Encoders]
\label{lemma:exist_opt_enc}
Consider a jointly observed set of views $\xb_V$, satisfying~\cref{assmp:general_assumption}.
Let $S_{k} \subseteq [N],\, k \in V$ be view-specific indexing sets of latent variables and define the shared coordinates $C := \bigcap_{k \in V} S_k$. For any content encoders $G := \{g_k:\Xcal_k\to (0, 1)^{|C|}\}_{k \in V}$~\pcref{defn:content_encoders}, we define the following objective:
\begin{equation}
    \label{eq:loss_set}
    \Lcal 
    \left(G\right)=
    \sum_{\substack{k, k^\prime \in V \\ k < k^\prime}}
    \EE
    \left[\norm{g_{k}(\xb_{k})-g_{k^\prime}(\xb_{k^\prime})}_2\right]- \sum_{k \in V} H\left(g_{k}(\xb_k)\right)
\end{equation}
where the expectation is taken with respect to $p(\xb_V)$ and where $H(\cdot)$ denotes differential entropy. Then the global minimum of the loss~\pcref{eq:loss_set} is lower bounded by zero, and there exists a set of content encoders~\cref{defn:content_encoders} which obtains this global minimum.
\end{lemma}
\begin{proof}
    Consider the objective function $\Lcal(G)$ defined in~\cref{eq:loss_set}, the global minimum of $\Lcal(G)$ is obtained when the first term (alignment) is minimized and the second term (entropy) is maximized. 
    The alignment term is minimized to zero when $g_{k}$ are perfectly aligned for all $k \in V$, i.e., $g_{k}(\xb_{k}) = g_{k^\prime}(\xb_{k^\prime}) $ for all $\xb_V \sim p_{\xb_V}$. 
    The second term (entropy) is maximized to zero \emph{only} when $g_{k}(\xb_k)$ is uniformly distributed on $(0, 1)^{|C|}$ for all views $k \in V$.
    
    To show that there exists a set of smooth functions: $G :=\{g_k\}_{k \in V}$ that minimizes $\Lcal(G)$, 
    we consider the inverse function of the ground truth mixing function ${f^{-1}_k}_{1:|C|}$, w.l.o.g. we assume that the content variables are at indices $1:|C|$.
    This inverse function exists and is a smooth function given by~\cref{assmp:general_assumption}(i) that each mixing function $f_k$ is a smooth invertible function. By definition, we have ${f^{-1}_k}_{1:|C|} (\xb_{k}) = \zb_C$ for $k \in V$.
    
    Next, we define a function $\boldsymbol{d}$ using \emph{Darmois construction}~\citep{darmois1951analyse} as follows:
    \begin{equation}
        d^{j}\left(\zb_C\right) := F_j\left( z_j | \zb_{1: j-1}\right) \qquad j \in \{1, \dots, |C|\},
    \end{equation}
    where $F_j$ denotes the conditional cumulative distribution function (CDF) of $z_j$ given $\zb_{1:j-1}$, i.e. $F_j\left( z_j | \zb_{1: j-1}\right) := \mathbb{P}\left(Z_j \leq z_j | \zb_{1:j-1}\right)$. 
    By construction, $\mathbf{d}\left(\zb_C\right)$ is uniformly distributed on $(0, 1)^{|C|}$. Moreover, $\db$ is smooth because $p_\zb$ is a smooth density by~\cref{assmp:general_assumption}(ii) and because conditional CDF of smooth densities is smooth %

    Finally, we define
    \begin{equation}
        \label{eq:optimal_gk}
        g_k := \mathbf{d} \circ {f^{-1}_k}_{1:|C|}: \Xcal_k \to (0, 1)^{|C|}, \quad k \in V,
    \end{equation}
    which is a smooth function as a composition of two smooth functions.

    Next, we show that the function set $G$ as constructed above attains the global minimum of $\Lcal(G)$. 
    Given that ${f^{-1}_{k}}_{1:|C|}(\xb_{k}) = {f^{-1}_{k^\prime}}_{1:|C|}(\xb_{k^\prime}) = \zb_C, \, \forall k, k^\prime \in V$, we have:
    \begin{equation}
    \begin{aligned}
        \Lcal 
        \left(G\right)=&
        \sum_{\substack{k, k^\prime \in V \\ k < k^\prime}}
        \EE
        \left[\norm{g_{k}(\xb_{k})-g_{k^\prime}(\xb_{k^\prime})}_2\right]- \sum_{k \in V} H\left(g_{k}(\xb_k)\right)\\[.5em]
        =& \sum_{\substack{k, k^\prime \in V \\ k < k^\prime}}
        \EE
        \left[\norm{\mathbf{d}\left(\zb_C\right)-\mathbf{d}\left(\zb_C\right)}_2\right]
        - \sum_{k \in V} H\left(\mathbf{d}\left(\zb_C\right)\right)
        \\[.5em]
        =& 0,
    \end{aligned}
    \end{equation}
    where $\zb_C$ is the shared content variables thus the first term (alignment) equals zero; and since $\mathbf{d}\left(\zb_C\right)$ is uniformly distributed on $(0, 1)^{|C|}$, the second term (entropy) is also zero.

    To this end, we have shown that there exists a set of smooth encoders $G := \{g_k\}_{k \in V}$ with $g_k$ as defined in~\cref{eq:optimal_gk} which minimizes the objective $\Lcal(G)$ in~\cref{eq:loss_set}.
\end{proof}

\begin{lemma}[Conditions of Optimal Encoders]
\label{lemma:cond_opt_enc}
    Assume the same set of views $\xb_V$ as introduced in Lemma~\ref{lemma:exist_opt_enc}, then for any set of smooth encoders $G:=\{g_k: \Xcal_{k} \to (0, 1)^{|C|}\}_{k \in V}$ to obtain the global minimum (zero) of the objective $\Lcal(G)$ in~\cref{eq:loss_set}, the following two conditions have to be fulfilled:
\begin{itemize}
    \item \textbf{Invariance}: All extracted representations $\hat{\zb}_{k} := g_k(\xb_k)$ must align across the views from the set $V$ almost surely:
    \begin{equation}
        \label{eq:cond1}
        g_k(\xb_k) = g_{k^\prime}(\xb_{k^\prime}) \quad  \forall k, k^\prime \in V \quad a.s.
    \end{equation}
    \item \textbf{Uniformity}: All extracted representations $\hat{\zb}_{k} := g_k(\xb_k)$ must be uniformly distributed over the hyper-cube $(0, 1)^{|C|}$.
\end{itemize}
\end{lemma}
\begin{proof}
    Given that $G = \argmin \Lcal(G)$, we have by Lemma~C.1:
    \begin{equation}
    \label{eq:min_G_L+1}
        \Lcal 
        \left(G\right)=
        \sum_{k, k^\prime \in V}
        \EE
        \left[\norm{g_{k}(\xb_{k})-g_{k^\prime}(\xb_{k^\prime})}_2\right]- \sum_{k \in V} H\left(g_{k}(\xb_{k})\right) = 0
    \end{equation}
    
    The minimum $L(G) = 0$ leads to following conditions:
    \begin{align}
        \EE\left[\norm{g_{k}(\xb_{k})-g_{k^\prime}(\xb_{k^\prime})}_2\right] &= 0 \quad \forall k, k^\prime \in V, k < k^\prime \label{eq:set_knvariance}\\
        H\left(g_{k}(\xb_{k})\right) &= 0 \quad \forall k \in V \label{eq:unif_k}
    \end{align}
    where~\cref{eq:set_knvariance} indicates the invariance condition holds for all views $x_{k}$ and smooth encoders $g_{k} \in G$ almost surely; and~\cref{eq:unif_k} implies that the encoded information $g_{k}(\xb_{k})$ must be uniformly distributed on $(0, 1)^{|C|}$.
\end{proof}

\begin{lemma}[Content-Style Isolation from Set of Views]
\label{lemma:content_style_isolation_set_of_views}
Assume the same set of views $\xb_{V}$ as introduced in Lemma~\ref{lemma:exist_opt_enc}, then for any set of smooth encoders $G:=\{g_k: \Xcal_{k} \to (0, 1)^{|C|}\}_{k \in V}$ that satisfies the \textbf{Invariance} condition~\pcref{eq:cond1}, the learned representation can only be dependent on the content variables $\zb_C := \{\zb_j : j \in C\}$, not any style variables $\zb^{\mathrm{s}}_k:= \zb_{S_k {\setminus} C}$ for all $k \in V$.
\end{lemma}
\begin{proof}
    Note that the learned representation can be rewritten as:
    \begin{equation}
        g_k(\xb_k) = g_k(f_k(\zb_{S_k})) \quad k \in V,
    \end{equation}
    we define
    \begin{equation}
        h_k := g_k \circ f_k \quad k \in V.
    \end{equation}
    Following the second step of the proof from~\citet[][Thm.~4.2]{von2021self}, we show by contradiction that both $h_{k}(\zb_{S_{k}})$ for all $k \in V$ can only depend on the shared content variables $\zb_C$.

    Let $k \in V$ be any view from the jointly observed set, suppose \emph{for a contradiction} that $h_k^{\mathrm{c}} := h_{k}(\zb_{S_k})_{1:|C|}$ depends on some component $z_q$ from the view-specific latent variables $\zb_{k}^\mathrm{s}$:
    \begin{equation}
        \label{eq:contradiction_assumption}
        \exists q \in \{1, \dots, \dimension(\zb_{k}^{\mathrm{s}})\}, \,\zb_{S_k} = (\zb_C^*, \zb_{k}^{\mathrm{s}*}) \in \Zcal_{k}, \quad s.t. \quad \dfrac{\partial h_{k}^{\mathrm{c}}}{\partial z_q}(\zb_C^*, \zb_{k}^{\mathrm{s}*}) \neq 0,
    \end{equation}
    which means that partial derivative of $h_{k}^{\mathrm{c}}$ w.r.t. some latent variable $z_q \in \zb_{k}^{\mathrm{s}}$ is non-zero at some point $\zb_{S_{k}} = (\zb_C^*, \zb_{k}^{\mathrm{s}*}) \in \Zcal_{k}$.
    Since $h_k^{\mathrm{c}}$ is smooth, its first-order (partial) derivatives are continuous.
    By continuity of the partial derivatives, $\partialfrac{z_q}{h_1^{\mathrm{c}}}$ must be non-zero in a neighborhood of $(\zb_C^*, \zb_{k}^{\mathrm{s}*})$, i.e.,
    \begin{equation}
        \label{eq:monotonicity}
        \exists \eta > 0 \quad s.t. \quad z_q \to h_k^{\mathrm{c}} (\zb_C^*, \zb_{{k}_{-q}}^{\mathrm{s}*}, z_q) \quad \text{is strictly monotonic on } (z_q - \eta, z_q + \eta),
    \end{equation}
    where $\zb_{{k}_{-q}}^{\mathrm{s}*}$ denotes the remaining view-specific style variables except $z_q$.
    
    Next, we define an auxiliary function for each pair of views $(k, k^\prime)$ with $k, k^\prime \in V, k < k^\prime$: $\psi_{k, k^\prime}: \Zcal_{C} \times \Zcal_{S_k \setminus C} \times \Zcal_{S_{k^\prime} \setminus C} \to \RR_{\geq 0}$
    \begin{equation}
    \label{eq:aux_func_kkprime}
    \begin{aligned}
        \psi_{k, k^\prime}(\zb_C, \zb_{k}^{\mathrm{s}}, \zb_{k^\prime}^{\mathrm{s}}) 
        :&= 
        \left\lvert h_{k}^{\mathrm{c}}\left(\zb_C, \zb_{k}^{\mathrm{s}} \right) - h_{k^\prime}\left(\zb_C, \zb_{k^\prime}^{\mathrm{s}} \right)\right\rvert\\
        &=  \left\lvert h_{k}^{\mathrm{c}}\left(\zb_{S_{k^\prime}}\right) - h_{k^\prime}^{\mathrm{c}}\left(\zb_{S_{k^\prime}}\right) \right \lvert \geq 0.
    \end{aligned}
    \end{equation}
    
    Summarizing the pairwise auxiliary functions, we have $\psi: \Zcal_{C} \times \prod_{k \in V} \Zcal_{S_k \setminus C} \to \RR_{\geq 0}$ as follows:
    \begin{equation}
    \label{eq:aux_func}
    \begin{aligned}
        \psi(\zb_C, \{\zb_{k}^{\mathrm{s}}\}_{k \in V}) 
        :&= 
        \sum_{\substack{k, k^\prime \in V\\ k < k^\prime}} \left\lvert h_{k}^{\mathrm{c}}\left(\zb_C, \zb_{k}^{\mathrm{s}} \right) - h_{k^\prime}\left(\zb_C, \zb_{k^\prime}^{\mathrm{s}} \right)\right\rvert\\
        &= \sum_{\substack{k, k^\prime \in V\\ k < k^\prime}}  \left\lvert h_{k}^{\mathrm{c}}\left(\zb_{S_{k^\prime}}\right) - h_{k^\prime}^{\mathrm{c}}\left(\zb_{S_{k^\prime}}\right) \right \lvert \geq 0
    \end{aligned}
    \end{equation}
    To obtain a contradiction to the invariance condition in Lemma~\ref{lemma:cond_opt_enc}, it remains to show that $\psi$ from~\cref{eq:aux_func} is \emph{strictly positive} with a probability greater than zero w.r.t.~the true generating process $p$; in other words, there has to exist at least one pair of views $(k, k^\prime)$ s.t. $\psi_{k, k^\prime} > 0$ with a probability greater than zero regarding $p$.

    Since $q \in S_k \setminus C$, there exists at least one view $k^\prime \neq k$ s.t. $q \notin S_{k^\prime}$ (otherwise the content block $C$ would contain $q$). We choose exactly such a pair of views $k, k^\prime$.
    
    Depending whether there is a zero point $z_q^0$  of $\psi$ within the region $(z_q - \eta, z_q + \eta)$, there are two cases to consider:
    \begin{itemize}
        \item If there is no zero-point $z_q^0 \in (z_q - \eta, z_q + \eta)$ s.t. $\psi_{k, k^\prime}\left(\zb_C^*, (\zb_{{k}_{-q}}^{\mathrm{s}*}, z_q^0), \zb_{k^\prime}^{\mathrm{s}*} \right) = 0 $, then it implies
        \begin{equation}
            \psi_{k, k^\prime}\left(\zb_C^*, (\zb_{{k}_{-q}}^{\mathrm{s}*}, z_q), \zb_{k^\prime}^{\mathrm{s}*}\right) > 0
        \quad 
        \forall z_q \in (z_q - \eta, z_q + \eta).
        \end{equation}
        So there is an open set $A := (z_q - \eta, z_q + \eta) \subseteq \Zcal_q$ such that the equation $\psi$ in~\cref{eq:aux_func} is strictly positive.
        \item Otherwise, there is a zero point $z_q^0$ from the interval $(z_q - \eta, z_q + \eta)$ with
            \begin{equation}
        \psi_{k, k^\prime} \left(\zb_C^*, (\zb_{{k}_{-q}}^{\mathrm{s}*}, z_q^0), \zb_{k^\prime}^{\mathrm{s}*} \right) = 0 \qquad z_q^0 \in (z_q - \eta, z_q + \eta),
    \end{equation}
    then strict monotonicity from~\cref{eq:monotonicity} implies that $\psi_{k, k^\prime} > 0$ for all $z_q$ in the neighborhood of $z_q^0$, therefore:
    \begin{equation}
        \label{eq:positive_set}
        \psi(\zb_C, \{\zb_{k}^{\mathrm{s}}\}_{k \in V}) > 0
        \quad 
        \forall z_q \in A := (z_q - \eta, z_q^0) \cup (z_q^0, z_q + \eta).
    \end{equation}
    \end{itemize}
    Since $\psi$ is a sum of compositions of two smooth functions (absolute different of two smooth functions), $\psi$ is also smooth.
    Consider the open set $\RR_{> 0}$ and note that, under a continuous function, pre-images of open sets are \emph{always open}.
    For the continuous function $\psi$, its pre-image $\Ucal$ corresponds to an \emph{open set}:
    \begin{equation}
        \Ucal \subseteq \Zcal_{C} \times \prod_{k \in V}\Zcal_{S_k \setminus C}
    \end{equation}
    in the domain of $\psi$ on which $\psi$ is strictly positive.
    Moreover, since~\cref{eq:positive_set} indicated that for all $z_q \in A$, the function $\psi$ is strictly positive, which means:
    \begin{equation}
        \{\zb_C^*\} 
        \times 
        \prod_{k: q \in S_k \setminus C} \left(\{\zb_{{k}_{-q}}^{\mathrm{s}*}\} \times A \right) 
        \times 
        \prod_{k: q \notin S_k} \{\zb_{k}^{\mathrm{s}*}\}
        \subseteq
        \Ucal,
    \end{equation}
    hence, $\Ucal$ is \emph{non-empty}.

    Given by~\cref{assmp:general_assumption} (ii) that $p_{\zb}$ is smooth and fully supported ($p_{\zb} > 0$ almost everywhere), the non-empty set $\Ucal$ is also fully supported by $p_{\zb}$, which indicates:
    \begin{equation}
        \PP \left(\psi(\zb_C, \{\zb_{k}^{\mathrm{s}}\}_{k \in V})  > 0 \right) \geq \PP \left(\Ucal \right) > 0,
    \end{equation}
    where $\PP$ denotes the probability w.r.t. the true generative process $p$.

    According to Lemma~\ref{lemma:cond_opt_enc}, the invariance condition and uniformity conditions has to be fulfilled. To this end, we have shown that the assumption~\cref{eq:contradiction_assumption} leads to an contradiction to the invariance condition~\cref{eq:cond1}. 
    Hence, assumption~\cref{eq:contradiction_assumption} cannot hold, i.e., $h_k^{\mathrm{c}}$ does not depend on any view-specific style variable $z_q$ from $\zb_{k}^{\mathrm{s}}$. It is only a function of the shared content variables $\zb_C$, that is, $\hat{\zb}_k^{\mathrm{c}} = h_k^{\mathrm{c}}(\zb_C)$.
\end{proof}

We list \citet[][Proposition 5.]{zimmermann2021contrastive} for future use in our proof:
\begin{proposition}[Proposition 5 of~\citet{zimmermann2021contrastive}.]
    Let $\Mcal, \Ncal$ be simply connected and oriented $\Ccal^1$ manifolds without boundaries and $h: \Mcal \to \Ncal$ be a differentiable map. Further, let the random variable $\zb \in \Mcal$ be distributed according to $\zb \sim p(\zb)$ for a regular density function $p$, i.e., $0 < p < \infty$. If the push-forward $p_{\#h}(\zb)$ through $h$ is also a regular density, i.e., $0 < p_{\#h} < \infty$, then $h$ is a bijection.
    \label{prop:zimmermann_invertibility_uniformality}
\end{proposition}

\idset*
\begin{proof}
    Lemma~\ref{lemma:exist_opt_enc} verifies the existence of such a set of smooth encoders that obtains the global minimum of~\cref{eq:main_loss_set} zero; Lemma~\ref{lemma:cond_opt_enc} derives the invariance conditions and the uniformity that the learned representations $g_{k}(\xb_k)$ have to satisfy for all views $k \in V$. Based on the invariance condition~\cref{eq:cond1}, Lemma~\ref{lemma:content_style_isolation_set_of_views} shows that the learned representation $g_{k}(\xb_k), \, k \in V$  can only depend on the content block, not on any style variables, namely $g_k(\xb_k) = h_k(\zb_C)$ for some smooth function $h_k: \Zcal_C \to (0, 1)^{|C|}$.
    
    We now apply~\citet[][Proposition 5.]{zimmermann2021contrastive} to show that all of the functions $h_k, k \in V$ are bijections. Note that both $\Zcal_{C}$ and $(0, 1)^{|C|}$ are simply connected and oriented $\Ccal^1$ manifolds, and $h_k$ are smooth, thus differentiable, functions that map the intersection set of random variables $\zb_C$ from $\Ccal$ to $(0, 1)^{|C|}$. Given by~\cref{assmp:general_assumption}(ii) that $p_{\zb_C}$ and the push-forward function through $h_k$ (uniform distributions) are regular densities, we conclude that all $h_k$ are diffeomorphisms for all $k \in V$.

    Thus we have shown that any content set of encoders $G$ that minimizes $\Lcal(G)$~\pcref{eq:main_loss_set} can extract the ground-truth content variables $\zb_C$ from view $\xb_k \in \Xcal_k$ up to a bijection $h_k: \Zcal_{C} \to (0, 1)^{|C|}$:
    \begin{equation}
        g_k(\xb_k) = h_k(\zb_C),
    \end{equation}
    That is, shared content $\zb_C$ is block-identified by the content encoders $G = \{g_k\}_{k \in V}$.
\end{proof}

\looseness=-1\textbf{Remark on the proof technique for~\cref{thm:ID_from_sets}.}  \,
For~\cref{thm:ID_from_sets}, one could imagine an alternate proof by induction over the number of views, where the proofs by~\citet{von2021self,daunhawer2023identifiability} would be the base case. We opted for a direct proof technique as the induction proof may have been perhaps more intuitive at a high level but was significantly longer. Additionally, we present the current version because it would be generally more accessible as a more familiar proof technique.
\subsection{Proof for~\cref{thm:general_ID_from_sets_size_unknown}}
Our proof consists of the following steps:
\begin{enumerate}
    \item We show in~\cref{lemma:exist_opt_sel_enc} the loss~\cref{eq:loss-selector-encoder} is lower bounded by zero and construct optimal $R^*$~\pcref{defn:view_specific_encoders},
    $\Phi^*$~\pcref{defn:content_selectors},
    $T^*$~\pcref{defn:aux_transformations} that reach this lower bound;
    \item Next, we show in~\cref{lemma:general_ID_from_sets} that, if the content sizes $|C_i|$ are known for all $V_i \in \Vcal$, then any view-specific encoders, content selectors, and projections $(R, \Phi, T)$ that minimize the loss~\cref{eq:loss-selector-encoder}, block-identify the content variables $\zb_{C_i} $ for any $V_i \in \Vcal$, using similar steps as in the proof for~\cref{thm:ID_from_sets}.
    \item As the third step, we show that any minimizer $R$~\pcref{defn:view_specific_encoders},
    $\Phi$~\pcref{defn:content_selectors},
    $T$~\pcref{defn:aux_transformations} of~\cref{eq:loss-selector-encoder} also minimizes the information-sharing regularizer~\pcref{defn:info_sharing_reg}; and show that the optimal solution $(R^*, \Phi^*, T^*)$ we constructed in the first step reaches this lower bound of~\cref{defn:info_sharing_reg}.
    \item Then, we show \emph{by contradiction} that any optimal content selector $\Phi^*$ that solves the constrained optimization problem in~\cref{eq:main_loss_set_size_unknown} recovers the correct content size $|C_i|$ for each subset $V_i$, using the invariance condition in~\cref{lemma:cond_opt_selector_enc}.
    \item Lastly, we apply the results from~\cref{lemma:general_ID_from_sets} and conclude our proof for~\cref{thm:general_ID_from_sets_size_unknown}.
\end{enumerate}
We rephrase each step as a separate lemma and use them to complete the final proof for~\cref{thm:general_ID_from_sets_size_unknown}.

\begin{lemma}[Existence of Encoders, Selectors and Projections]
\label{lemma:exist_opt_sel_enc}
    Consider a jointly observed set of views $\xb_{V}$ satisfying~\cref{assmp:general_assumption}. For any set of view-specific encoders $R$~\pcref{defn:view_specific_encoders}, content selectors $R_{\Phi}$~\pcref{defn:content_selectors} and projections $T$~\pcref{defn:aux_transformations}, we define the following objective:
    \begin{equation}
        \label{eq:loss-selector-encoder}
        \begin{aligned}
         \Lcal
        \left(R, \Phi, T \right) 
        = 
        \sum_{V_i \in \Vcal} \sum_{\substack{k, k^\prime \in V_i \\ k < k^\prime}} 
        \EE
        \left[\norm{\phi^{(i, k)} \oslash r_{k}(\xb_{k}) - \phi^{(i, k^\prime)}\oslash r_{k^\prime}(\xb_{k^\prime})}_2\right]
        - 
         \sum_{k \in V}H\left(t_k \circ r_{k}(\xb_k)\right).
        \end{aligned}
  \end{equation}
  which is lower bounded by zero; and there exists such combination of $R, \Phi, T$ that obtains this global minimum zero.
\end{lemma}
\begin{proof}
    Consider the objective function $\Lcal(R, \Phi, T)$~\pcref{eq:loss-selector-encoder}, the global minimum of $\Lcal(R, \Phi, T)$ is obtained when the first term (alignment) is minimized and the second term (entropy) is maximized. 
    The alignment term is minimized to zero when selected representations $\phi^{(i, k)} \oslash r_{k}$ are perfectly aligned for all $k \in V$ almost surely.
    The second term (entropy) is maximized to zero \emph{only} when $t_k \circ r_{k}(\xb_k)$ is uniformly distributed on $(0, 1)^{|S_k|}$ for all view $k \in V$. Thus we have shown that the loss~\pcref{eq:loss-selector-encoder} is lower-bounded by zero.
        
    The optimal view-specific encoders can be defined via the inverse of the view-specific mixing functions $\{f_k\}_{k \in V}$, which by~\cref{assmp:general_assumption}(i) are smooth and invertible.  By definition, we have ${f^{-1}_{k}}(\xb_{k}) = \zb_{S_k}$ for all $k \in V$. Formally, we define the set of optimal view-specific encoders 
    \begin{equation}
    \label{eq:opt_view_specific_encoder}
        R := \{{f^{-1}_{k}}\}_{k \in V}.
    \end{equation}

    Next, we define the optimal auxiliary transformation $t_k$ for each view $k$ using \emph{Darmois construction}, writing $t_k \circ r_k(\xb_k) = t_k \circ f^{-1}_k(\xb_k) = t_k^{j}\left({\zb}_{S_k}\right)$, we have:
    \begin{equation}
    \label{eq:opt_aux_transformations}
        t_k^{j}\left({\zb}_{S_k}\right) := F^k_j\left(
        [{\zb}_{S_k}]_j | [{\zb}_{S_k}]_{1:j-1}\right) = \mathbb{P}\left([Z_{S_k}]_j \leq [{\zb}_{S_k}]_j | [{\zb}_{S_k}]_{1:j-1}\right) \quad j \in \{1, \dots, |S_k|\},
    \end{equation}
    where $F^k_j$ denotes the conditional cumulative distribution function (CDF) of $[\zb_{S_k}]_j$ given $[\zb_{S_k}]_{1:j-1}$. Thus, $t_k \left({\zb}_{S_k}\right)$ is uniformly distributed on $(0, 1)^{|S_k|}$ and $t_k$ is smooth by~\cref{assmp:general_assumption}(ii) which states that $p_\zb$ is a smooth density.

    As for the optimal content selectors $\Phi = \{\phi^{(i, k)}\}_{V_i \in \Vcal, k \in V_i}$,
    choose $\phi^{(i, k)}$ 
    such that
    \begin{equation}
        \label{eq:opt_selectors}
        \phi^{(i, k)} \oslash \hat{\zb}_{S_k} := \hat \zb_{C_i}
    \end{equation}
    Writing ${f^{-1}_{k}}({\xb_k}) = \zb_{S_k}$, the loss $\Lcal(R, \Phi, T)$ from~\cref{eq:loss-selector-encoder} takes the value:
    \begin{equation}
    \begin{aligned}
        \Lcal
        \left(R, \Phi, T\right) 
        =&
        \sum_{V_i \in \Vcal} \sum_{\substack{k, k^\prime \in V_i \\ k < k^\prime}} 
        \EE
        \left[\norm{\phi^{(i, k)}\oslash r_k(\xb_k) - \phi^{(i, k^{\prime})} \oslash r_{k^\prime}(\xb_{k^\prime})}_2\right]
        -
         \sum_{k \in V} H\left(t_k \circ r_{k}(\xb_k)\right)\\
        =&
        \sum_{V_i \in \Vcal} \sum_{\substack{k, k^\prime \in V_i \\ k < k^\prime}} 
        \EE
        \left[\norm{\phi^{(i, k)} \oslash f_k^{-1}(\xb_k) - \phi^{(i, k^{\prime})} \oslash f_{k^\prime}^{-1}(\xb_{k^\prime})}_2\right]
        -
         \sum_{k \in V} H\left(t_k \circ f_k^{-1}(\xb_k)\right)\\
         =&
        \sum_{V_i \in \Vcal} \sum_{\substack{k, k^\prime \in V_i \\ k < k^\prime}} 
        \EE
        \left[\norm{\phi^{(i, k)} \oslash\zb_{S_k} - \phi^{(i, k^{\prime})} \oslash \zb_{S_{k^\prime}}}_2\right]
        -
         \sum_{k \in V} H\left(t_k(\zb_{S_k})\right)\\
         =&
        \sum_{V_i \in \Vcal} \sum_{\substack{k, k^\prime \in V_i \\ k < k^\prime}} 
        \EE
        \left[\norm{\zb_{C_i} - \zb_{C_i}}_2\right]
        -
         \sum_{k \in V} H\left(t_k \left(\zb_{S_k}\right)\right) \\
         =& 0
    \end{aligned}
    \end{equation}
    Note that the first term is minimized to zero because the shared content values $\zb_{C_i}$ align among the views in one subset $V_i \in \Vcal$; the second term is maximized to zero because $t_k \left(\zb_{S_k}\right)$ is uniformly distributed on $(0, 1)^{|S_k|}$ given by the property of \emph{Darmois construction}~\citep{darmois1951analyse}.
To this end, we have shown that there exists such optimum $R, \Phi, T$ as defined in~\cref{eq:opt_view_specific_encoder,eq:opt_selectors,eq:opt_aux_transformations} that minimizes the objective in~\cref{eq:loss-selector-encoder}.
\end{proof}

\begin{lemma}[Conditions of Optimal Encoders, Selectors and projections]
\label{lemma:cond_opt_selector_enc}
    Given the same set of views $\xb_{V}$ as introduced in~\cref{lemma:exist_opt_sel_enc}, to minimize $\Lcal(R, \Phi, T)$ in~\cref{eq:loss-selector-encoder}, any optimum $R, \Phi, T$~\pcref{defn:view_specific_encoders,defn:content_selectors,defn:aux_transformations} has to satisfy
    similar \textbf{invariance} and \textbf{uniformity} conditions from~\cref{lemma:cond_opt_enc}:
    \begin{itemize}
    \item \textbf{Invariance}: All \textbf{selected} representations $\phi^{(i, k)} \oslash r_k(\xb_k), k \in V$ must align across the views from the set $V_i \in \Vcal$ almost surely:
    \begin{equation}
        \phi^{(i, k)} \oslash r_k(\xb_k)= \phi^{(i, k^\prime)} \oslash r_{k^\prime}(\xb_{k^\prime}) \quad  \forall V_i \in \Vcal \, \forall k, k^\prime \in V_i \quad a.s.
    \end{equation}
    \item \textbf{Uniformity}: All extracted representations $t_k \circ r_k(\xb_k), k \in V$ must be uniformly distributed over the hyper unit-cube $(0, 1)^{|S_k|}$.
\end{itemize}
\end{lemma}
\begin{proof}
    The minimum of $\Lcal(R, \Phi, T) = 0$ can only be obtained when both terms are zero. For the first term (alignment) to be zero, it is necessary that $ \phi^{(i, k)} \oslash r_k(\xb_k)= \phi^{(i, k^\prime)} \oslash r_{k^\prime}(\xb_{k^\prime})$ almost surely for all $V_i \in \Vcal\, , k, k^\prime \in V_i$ w.r.t.~the true generating process.
    The second term (entropy) is upper-bounded by zero; this maximum can only be obtained when the auxiliary encoding $t_k \circ r_k(\xb_k), k \in V$ follows \textbf{\emph{uniformity}}, as also indicated by Lemma~\ref{lemma:cond_opt_enc}.
\end{proof}

\begin{lemma}[View-Specific Encoder for Identifiability Given Content Sizes] 
\label{lemma:general_ID_from_sets}
Consider a jointly observed set of views $\xb_{V}$ satisfying~\cref{assmp:general_assumption} and assume that the dimensionality of the \textbf{subset-specific} content $|C_i|$ is given for all subset $V_i \in \Vcal$. We consider a special type of content selectors $\Phi$ with $\norm{\phi^{(i, k)}}_0 = |C_i|\,$ for all $k \in V_i$.
Let $R, T$ respectively denote some view-specific encoders~\pcref{defn:view_specific_encoders}, and projections~\pcref{defn:aux_transformations}, which
jointly minimize the following objective together with the special content selectors $\Phi$:
\begin{equation}
\label{eq:loss_set_size_known}
    \begin{aligned}
     \Lcal
    \left(R, \Phi, T\right) 
    = 
   \sum_{V_i \in \Vcal} \sum_{\substack{k, k^\prime \in V_i \\ k < k^\prime}} 
    \EE
    \left[\norm{\phi^{(i, k)} \oslash r_{k}(\xb_{k}) - \phi^{(i, k^\prime)}\oslash r_{k^\prime}(\xb_{k^\prime})}_2\right] 
    - 
     \sum_{k \in V}H\left(t_k \circ r_{k}(\xb_k)\right).
    \end{aligned}
  \end{equation}

 Then for any view $k \in V$, any subset of views $V_i \in \Vcal$ with $k \in V_i$, the composed function $\phi^{(i, k)} \oslash r_k$ block-identifies the shared \textbf{content} variables $\zb_{C_i}$ in the sense that
the learned representation $\hat{\zb}_k^{(i)} := \phi^{(i, k)} \oslash r_{k}(\xb_k)$ is related to the ground truth content variables through some smooth invertible mapping $h_{k}:\Zcal_{C_{i}}\to\Zcal_{C_i}$ with $\hat{\zb}_k^{(i)} = h_k^{(i)} \left(\zb_{C_i}\right)$.
\end{lemma}

\begin{proof}
~\cref{lemma:exist_opt_sel_enc} verifies that there exists such optimum which minimizes the loss~\cref{eq:loss_set_size_known} to zero; the invariance and uniformity conditions have to be satisfied by any optimum, as shown in Lemma~\ref{lemma:cond_opt_selector_enc}. Following ~\cref{lemma:content_style_isolation_set_of_views}, the composition $r_k^{(i)}:= \phi^{(i, k)} \oslash r_{k}$ can only encode information related to the subset-specific content $C_i$ for any subset $V_i \in \Vcal$ otherwise it will lead to a contradiction to the invariance condition from~\cref{lemma:cond_opt_selector_enc}. The last step is to prove the invertibility of the encoders $G$. Notice that
\begin{equation*}
    t_k \circ r_k(\xb_k) = t_k \circ r_k \circ f_k(\zb_{S_k})
\end{equation*}
By applying \citet[][Proposition 5.]{zimmermann2021contrastive} with similar arguments as in the proof for~\cref{thm:ID_from_sets}, we can show that composition $t_k \circ r_k \circ f_k$ is a smooth bijection of the subset-specific content $\zb_{C_i}$. Since $f_k$ is a smooth invertible mapping by~\cref{assmp:general_assumption}~(i),  we have:
\begin{equation*}
(t_k \circ r_k \circ f_k) \circ f_k^{-1} = (t_k \circ r_k) \circ (f_k \circ f_k^{-1}) = t_k \circ r_k, 
\end{equation*}
Hence, $t_k \circ r_k$ is bijective as the composition of bijections is a bijection. Next, we show that $r_k$ is bijective. Showing that $r_k$ is bijective on its image is equivalent to showing that it is injective. By contradiction, suppose $r_k$ is not injective. Thus there exists distinct values $\xb^1_k, \xb^2_k \in \mathcal{X}_k$ s.t. $r_k(\xb^1_k) = r_k(\xb^2_k)$. This implies that $t_k \circ r_k(\xb^1_k) = t_k \circ r_k(\xb^2_k)$, which violate injectivity of $t_k \circ r_k$. Thus, $r_k$ must be injective.

To this end, we conclude that any $R, \Phi, T$ that minimizes~\cref{eq:loss_set_size_known} block-identifies the shared content variables $\zb_{C_i}$ for any subset of views $V_i \in \Vcal$.
\end{proof}

\begin{claim}
    \label{claim:optimal_Rin_reg}
   For any $(R, \Phi, T)~\pcref{defn:view_specific_encoders,defn:content_encoders,defn:aux_transformations}$ that minimizes the loss~\cref{eq:loss-selector-encoder}, the $\mathrm{Reg}(\Phi)$~\pcref{defn:info_sharing_reg} is lower bounded by $- \sum_{V_i \in \Vcal} |C_i| \cdot |V_i|$ and this minimum is obtained at the optimal content selectors defined in~\cref{eq:opt_selectors}.
\end{claim}

\begin{proof}
    Suppose \emph{for a contradiction} that there exists some binary weight parameters $\tilde{\Phi} \neq \Phi$ with
    \begin{equation}
        \label{eq:contradiction_optimal_R_noVin_reg}
        \mathrm{Reg}(\tilde{\Phi}) = - \sum_{V_i \in \Vcal}  \sum_{k \in V_i}
    \norm{\tilde{\phi}^{(i, k)}}_0  < \mathrm{Reg}(\Phi),
    \end{equation}
    which means, there exists at least one vector $\tilde{\phi}^{(i, k)}$ for some view $k \in V$, subset $V_i \in \Vcal$, such that
    \begin{equation}
        \tilde{\phi}^{(i, k)} \oslash r_k(\xb_k) = \hat{\zb}_A \qquad |A| > |C_{i}|,
    \end{equation}
    where $A \subseteq S_k$ is an index subset of the view-specific latents $S_k$. Given that $R, \Phi,  T$ minimizes $\Lcal(R, \Phi,  T)$ from ~\cref{eq:loss-selector-encoder}, these minimizers have to satisfy the invariance and uniformity constraint as shown in~\cref{lemma:cond_opt_selector_enc}. Since uniformity implies invertibility~\citep{zimmermann2021contrastive}, the learned representation $r_k(\xb_k)$ contains sufficient information about the original view $\xb_k$ s.t. the view $\xb_k$ can be reconstructed by some decoder given enough capacity. Given that the number of selected dimensions $|A| > |C_i|$, at least one latent component $j \in A$ will contain information that is not jointly shared by $V_i$. That means the composition $r_k^{(i)}:= \phi^{(i, k)} \oslash r_{k}$ encodes some information other than just content $C_i$. As shown in~\cref{lemma:content_style_isolation_set_of_views}, any dependency from the learned representation on non-content variables leads to contradiction to the invariance condition as derived in~\cref{lemma:cond_opt_selector_enc}. Therefore, the optimal content selectors $\Phi$ following the definition in~\cref{eq:opt_selectors} must obtain the global minimum of the information-sharing regularizer~\pcref{defn:info_sharing_reg}, which equals $- \sum_{V_i \in \Vcal}  \sum_{k \in V_i} |C_i|$.
\end{proof}
\unifEnc*
\begin{proof}
    Lemma~\ref{lemma:exist_opt_sel_enc} confirms that
    there exist view-specific encoders $R$, 
    content selectors $\Phi$, 
    and projections $T$ that obtain the minimum of the unregularized loss~\cref{eq:loss-selector-encoder} (equals zero);
    Additionally, any optimal $R, \Phi, T$ fulfills the invariance condition and uniformity ~\pcref{lemma:cond_opt_selector_enc} s.t. they obtain the global minimum zero. Using the invariance condition, 
    Claim~\ref{claim:optimal_Rin_reg} substantiates that the optimal content selectors as defined in~\cref{eq:opt_selectors} minimizes the regularization term~\pcref{defn:info_sharing_reg}. 
    We have thus shown that with $R, \Phi, T$ (as defined in~\cref{eq:opt_view_specific_encoder,eq:opt_selectors,eq:opt_aux_transformations}),~
    \cref{eq:main_loss_set_size_unknown} obtains the global minimum. 

    Next, we show that the number of selected dimensions from each selector $\phi^{(i, k})$, i.e., the $L_0$ norm of $\phi^{(i, k)}$, align with the size of the shared content $|C_{i}|$.

    Among the content selectors that minimize the unregularized loss~\pcref{eq:loss-selector-encoder}, we consider some content selectors $\Phi^* \in \argmin \mathrm{Reg}(\Phi)$ that also minimize the information-sharing regulariser defined in~\cref{defn:info_sharing_reg}, that is:
    $$\mathrm{Reg}(\Phi^*) = - \sum_{V_i \in \Vcal}  \sum_{k \in V_i} |C_i|.$$
    Suppose \emph{for a contradiction} that 
    there exists a pair of binary selectors 
    $(\phi^{(i, k)}, \phi^{(i^{\prime}, k^\prime)})$ 
    with $\phi^{(i, k)} \in \{0, 1\}^{|S_k|}$ and $\phi^{(i^{\prime}, k^{\prime})} \in \{0, 1\}^{|S_{k^\prime}|}$ such that
    \begin{equation}
        \norm{\phi^{(i, k)}}_0 > |C_{i}|; \qquad \norm{\phi^{(i, k^\prime)}}_0 < |C_{i^\prime}|,
    \end{equation}
    which indicates that there exists at least one latent component $j \in S_k \setminus C_{i}$ being selected by $\phi^{(i, k)}$; similarly, this contradicts the invariance condition as shown in ~\cref{lemma:content_style_isolation_set_of_views}. %
    Hence, the number of dimensions selected by each $\phi^{(i, k)}$ has to equal the content size $|C_i|$.

    At this stage, the problem setup is reduced to the case in~\cref{lemma:general_ID_from_sets} where the size of the content variables $|C_i|$ are given for all subset of views $V_i \in \Vcal$. Hence, applying~\cref{lemma:general_ID_from_sets}, we conclude that any $R, \Phi, T$~\pcref{defn:view_specific_encoders,defn:content_selectors,defn:aux_transformations} that minimize~\cref{eq:main_loss_set_size_unknown} block-identify the shared content variables $\zb_{C_i}$ for any subset of views $V_i \in \Vcal$ and for all views $k \in V_i$.
\end{proof}
\subsection{Proofs for Identifiability Algebra}
Let $\zb_{C_1}, \zb_{C_2}$ be two sets of content variables indexed by  $C_1, C_2 \subseteq [N]$  that are block-identified by some smooth encoders $g_1: \Xcal_1 \to \Zcal_{C_1}, g_2: \Xcal_2 \to \Zcal_{C_2}$, then it holds for $C_1, C_2$ that:
\idAlgIntersection*
\begin{proof}
    By the definition of block-identifiability, we construct two synthetic views using the learned representation from $\xb_1$ and $\xb_2$:
    \begin{equation}
        \begin{aligned}
            \xb^{(1)} &:= g_1(\xb_1) = h_1(\zb_{C_1})\\
            \xb^{(2)} &:= g_2(\xb_2) = h_2(\zb_{C_2})
        \end{aligned}
    \end{equation}
    for some smooth invertible mapping $h_k : \Zcal_{C_k} \to \Zcal_{C_k} \, k \in \{1, 2\}$.
    Applying the~\cref{thm:ID_from_sets} with two views, we can block-identify the intersection $C_1 \cap C_2$ using this pair of views $(\xb^{(1)}, \xb^{(2)})$.
\end{proof}

\idAlgComplement*
\begin{proof}
     Construct the same synthetic views $\xb^{(1)}, \xb^{(2)}$ as in the proof for~\cref{cor:id_alg_intersection}. We then can consider the intersection $C_1 \cap C_2$ as the content variable and $C_1 \setminus C_2$ as the style variable from these two synthetic views $(\xb^{(1)}, \xb^{(2)})$. \emph{Private Component Extraction} from~\citet[Theorem 2.][]{lyu2021understanding} has shown that if the style variable is independent of the content, then the style variables can also be extracted up to a smooth invertible mapping. Therefore, we conclude that the complement $\zb_{C_1 \setminus C_2}$ can also be block-identified.
\end{proof}
\idAlgUnion*
\begin{proof}
    We rephrase $C_1{\cup}C_2$ as a union of the following disjoint parts:
    \begin{equation}
        C_1 {\cup} C_2 = (C_1 {\cap} C_2) {\cup} (C_1 {\setminus} C_2) {\cup} (C_2 {\setminus} C_1)
    \end{equation}
    Following the definition from~\cref{cor:id_alg_intersection,cor:id_alg_complement} have shown that:
    \begin{equation}
        \begin{aligned}
            \hat{\zb}_{\cap}&:= h_{\cap}(\zb_{C_1 {\cap} C_2})\\
            \hat{\zb}_{1{\setminus} 2} &:= h_{1{\setminus} 2}(\zb_{C_1{\setminus} C_2})\\
            \hat{\zb}_{2{\setminus} 1} &:= h_{2{\setminus} 1}(\zb_{C_2{\setminus} C_1}),
        \end{aligned}
    \end{equation}
    By concatenate the learned representations, we define $h_{\cup}: \Zcal_{C_1 \cup C_2} \to \Zcal_{C_1 \cup C_2}$ as
    \begin{equation}
         h_{\cup}(\zb_C, \zb_{1 {\setminus} 2}, \zb_{2 {\setminus} 1}) := [\hat{\zb}_C, \hat{\zb}_{1{\setminus} 2}, \hat{\zb}_{2{\setminus} 1}] =  h_{\cup}(\zb_{C_1 {\cup} C_2}),
    \end{equation}
    hence, the union $C_1 {\cap} C_2$ can be block-identified.
\end{proof}

\section{Experimental Results}
\label{sec:exp_details}
This section provides further details about the datasets and implementation details in~\cref{sec:experiment}. The implementation is built upon the code open-sourced by~\citet{zimmermann2021contrastive,von2021self,daunhawer2023identifiability}.

\subsection{Numerical Experiment -- Theory Validation}
\label{app:numerical}

\looseness=-1\textbf{Data Generation.} 
For completeness, we summarize the setting of our numerical experiments. We generate synthetic data following~\cref{example:intuitive}, which we also report below. The latent variables are sampled from a Gaussian distribution $\zb \sim \Ncal(0, \Sigma_{\zb})$, where possible \emph{causal} dependencies are encoded through $\Sigma_{\zb}$. Note that in this setting the ground truth causal variables will be related linearly to each other.
\begin{equation}
\textstyle
\begin{aligned}
     \xb_1 &= f_1(\zb_1, \zb_2, \zb_3, \zb_4, \zb_5),  \qquad
     \xb_2 = f_2(\zb_1, \zb_2, \zb_3, \zb_5, \zb_6), \\
     \xb_3 &= f_3(\zb_1, \zb_2, \zb_3, \zb_4, \zb_6), \qquad
     \xb_4 = f_4(\zb_1, \zb_2, \zb_4, \zb_5, \zb_6) .\quad
\end{aligned}
\end{equation}

\looseness=-1\textbf{Implementation Details.} 
We implement each view-specific mixing function $f_k$, for each view $k = 1, 2, 3, 4$, using a 3-layer 
\emph{invertible, untrainable} MLP~\citep{haykin1994neural} 
with LeakyReLU~\citep{xu2015empirical}($\alpha = 0.2$). 
The weight parameters in the mixing functions are \emph{randomly initialized}. 
For the \emph{learnable} view-specific encoders, 
we use a 7-layer MLP with LeakyReLU ($\alpha = 0.01$) 
for each view. 
The encoders are trained using the Adam optimizer~\citep{kingma2017adam} with \emph{lr=1e-4}. 
All implementation details are summarized in~\cref{tab:params_num}.

\looseness=-1\textbf{Additional Experiments.}
We experiment on \emph{causally dependent} synthetic data, generated by $\zb \sim \Ncal(0, \Sigma_{\zb})$ with $\Sigma_\zb \sim \mathrm{Wishart}(0, I)$. The results are shown in \cref{fig:numerical_heatmap_dep}. 
The rows denote the ground truth latent factors, 
and the columns represent the learned representation from different subsets of views. Each cell reports the $R^2$ score between the respective ground truth factors and the learned representation. 
For example, the cell with col=$\{\xb_1, \xb_2\}$ and row=$\zb_1$ shows the $R^2$ score when trying to predict $\zb_1$ using the learned representation from subset $\{\xb_1, \xb_2\}$. 
Since dependent style variables become predictable, as discussed in~\cref{app:numerical}, 
we aim to verify that the learned representation contains \emph{all and only} content variables. 
In other words, 
it \emph{block-identifies} the ground truth content factors. For that, we consider all the views $\{\xb_1, \dots, \xb_4\}$ and train a linear regression from the \emph{ground truth content variables} $\zb_1, \zb_2$ to the individual style variables $\zb_3, \zb_4, \zb_5, \zb_5$. 
We report the coefficient of determination $R^2$ in~\cref{tab:linear_greg_GT_c2s}. 
We observe that the $R^2$ values obtained from the ground truth content are highly similar to the ones in the last column of the heatmap~\pcref{fig:numerical_heatmap_dep}. 
Based on this, 
we have showcased that the learned representation indeed \emph{block-identifies} the content variables.

\looseness=-1 \textbf{Additional Evaluation Metric.} We report the Mean Correlation Coefficient(MCC)~\citep{khemakhem2020ice} on the numerical experiments. MCC has been used in several recent works on identifiability of causal representation learning~\citep{buchholz2023learning,von2023nonparametric}, it measures the \emph{component-wise linear} correlation up to permutations. A high MCC (close to one) indicates a clear 1-to-1 linear correspondence between the learned representation and the ground truth latents. We remark our theoretical framework considers block-identifiability, which could imply any type of bijective relation to the ground truth content variables, including nonlinear transformations. Nevertheless, we observe high MCC score on both independent and dependent cases, showing that the learned representation having a high \emph{linear} correlation to the latent components which indicates stronger identifiability results.
\begin{table}
    \centering
       \caption{{\textbf{Linear Evaluation}: Mean Correlation Coefficients across multiple views.}}
    \resizebox{\linewidth}{!}{%
    \begin{tabular}{lccccccccccc}
    \toprule
      \rowcolor{Gray!20}  & $(\mathbf{x}_1, \mathbf{x}_2)$ & $(\mathbf{x}_1, \mathbf{x}_3)$ & $(\mathbf{x}_1, \mathbf{x}_4)$ & $(\mathbf{x}_2, \mathbf{x}_3)$ & $(\mathbf{x}_2, \mathbf{x}_4)$ & $(\mathbf{x}_3,\mathbf{x}_4)$ & $(\mathbf{x}_1,\mathbf{x}_2,\mathbf{x}_3)$ & $(\mathbf{x}_1,\mathbf{x}_2,\mathbf{x}_4)$ & $(\mathbf{x}_1,\mathbf{x}_3,\mathbf{x}_4)$ & $(\mathbf{x}_2,\mathbf{x}_3,\mathbf{x}_4)$ & $(\mathbf{x}_1,\mathbf{x}_2,\mathbf{x}_3,\mathbf{x}_4)$  \\
      \midrule
ind.              & $0.887 \pm 0.000$              & $0.881 \pm 0.000$              & $0.882 \pm 0.000$              & $0.885 \pm 0.000$              & $0.886 \pm 0.000$              & $0.880 \pm 0.000$             & $0.853 \pm 0.000$                          & $0.854 \pm 0.000$                          & $0.846 \pm 0.000$                          & $0.851 \pm 0.000$                          & $0.786 \pm 0.000$                                       \\ 
dep.                & $0.956 \pm 0.000$              & $0.880 \pm 0.002$              & $0.891 \pm 0.002$              & $0.795 \pm 0.002$              & $0.805 \pm 0.002$              & $0.805 \pm 0.002$             & $0.945 \pm 0.001$                         & $0.969 \pm 0.001$                          & $0.858 \pm 0.003$                          & $0.744 \pm 0.003$                          & $0.944 \pm 0.001$ \\
      \bottomrule
    \end{tabular}
    }
\end{table}

\begin{table}
    \centering
    \begin{minipage}[t]{.4\textwidth}
        \caption{Parameters for numerical simulation~\pcref{subsec: num_results,app:numerical}.}
        \label{tab:params_num}
        \begin{tabular}{ll}
        \toprule
         \rowcolor{Gray!20} \textbf{Parameter}   & \textbf{Value} \\
         \midrule
          Mixing function   & 3-layer MLP\\
          Encoder & 7-layer MLP\\
          Optimizer & Adam\\
          Adam: learning rate& 1e-4\\
          Adam: beta1 & 0.9\\
          Adam: beta2 & 0.999\\
          Adam: epsilon & 1e-8\\
          Batch size& 4096\\
          Temperature $\tau$ & 1.0\\
          \# Iterations &100,000\\
          \# Seeds & 3\\
          Similarity metric & Euclidian\\
          \bottomrule
        \end{tabular}
    \end{minipage}
    \hfill
    \begin{minipage}[t]{.58\textwidth}
    \centering
       \caption{Parameters for experiments~\cref{subsec:case1,subsec:case2,app:c3di,app:m3di}. $^\ast$: for both image and text encoders. $^{\ast\ast}$: hyper-arapmeter for BarlowTwins~\citep{zbontar2021barlow}.}
           \label{tab:param_cm3di}
    \begin{tabular}{ll}
    \toprule
      \rowcolor{Gray!20} \textbf{Parameter}   &  \textbf{Values} \\
      \midrule
      Content encoding size$^\ast$   &  8 \\
      View-specific encoding size$^\ast$  &  11 \\
      Image hidden size & 100 \\
      Text embedding dim & 128 \\
      Text vocab size & 111\\
      Text fbase & 25\\
      Batch size & 128\\
      Temperature & 1.0\\
      Off-diagonal constant $\lambda^{\ast\ast}$ & 1.0\\
      Optimizer & Adam\\
      Adam: beta1 & 0.9\\
      Adam: beta2 & 0.999\\
      Adam: epsilon & 1e-8\\
      Adam: learning rate & 1e-4\\
      \# Iterations &300,000\\
      \# Seeds & 3\\
      Similarity metric & Cosine similarity\\
      Gradient clipping & 2-norm; max value 2\\
      \bottomrule
    \end{tabular}
    \end{minipage}
\end{table}

\begin{table}
\begin{minipage}[t]{.5\textwidth}
    \centering
    \includegraphics[width=\textwidth]{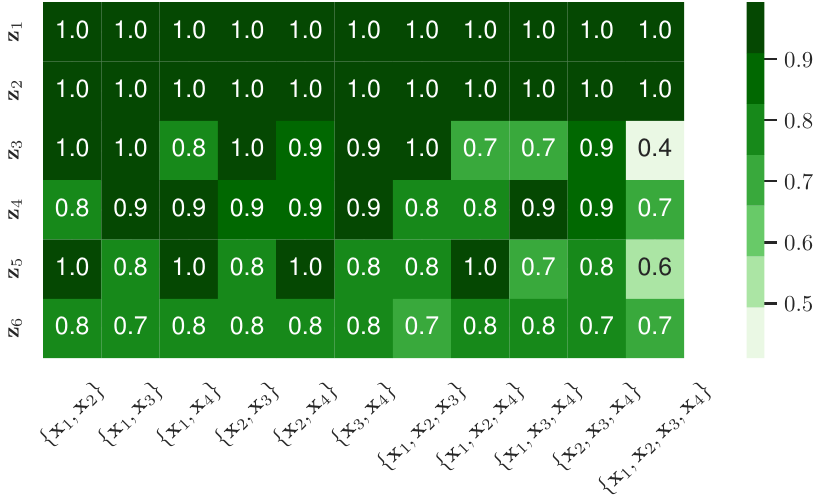}
    \captionof{figure}{\textbf{Theory Verfication:} Average $R^2$ across multiple views generated from \emph{causally dependent} latents.}
    \label{fig:numerical_heatmap_dep}
\end{minipage}
\hfill
\begin{minipage}[b]{0.45\textwidth}
    \centering
    \caption{\textbf{Linear $R^2$ from \emph{ground truth} content variables to styles} when consider $\{\xb_1, \xb_2, \xb_3, \xb_4\}$, these values align with the last column of~\cref{fig:numerical_heatmap_dep}, showing that we have \emph{block-identified} the content variables $\{\zb_1, \zb_2\}$}
    \label{tab:linear_greg_GT_c2s}
    \begin{tabular}{ccccc}
    \toprule
        \rowcolor{Gray!20} \textbf{content} & \multicolumn{4}{c}{\textbf{style}}  \\
         \rowcolor{Gray!20}$\{\zb_1, \zb_2\}$ & $\zb_3$ & $\zb_4$ & $\zb_5$ & $\zb_6$\\
        \midrule
         1.0 & 0.32 & 0.65 & 0.58 & 0.71\\
         \bottomrule
    \end{tabular}
\end{minipage}
\end{table}

\subsection{Self-Supervised Disentanglement}
\looseness=-1\textbf{Datasets.} \label{app:ss_dis}
In this experiment, 
we test on \emph{MPI-3D complex}~\citep{gondal2019transfer}
and 
\emph{3DIdent}~\citep{zimmermann2021contrastive}. 
Both are high-dimensional image datasets generated from 
\emph{mutually independent latent factors}: 
\emph{MPI-3D complex} contains real-world complex shaped images with ten \emph{discretized} latent factors while 
\emph{3DIdent}
renders a teapot conditioned on ten \emph{continuous} latent factors.

\begin{figure}[htbp]
    \begin{subfigure}[b]{.45\linewidth}
        \centering
        \includegraphics[width=\linewidth]{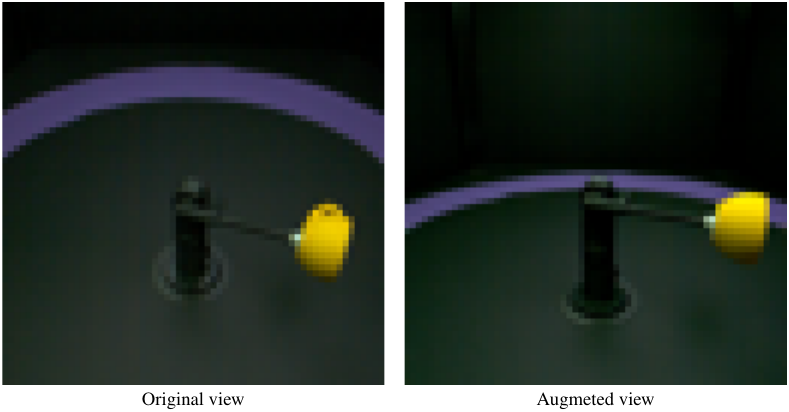}
        \caption{\textbf{Example Input}: \emph{MPI-3D complex}}
        \label{fig:input_mpi3d_complex}
    \end{subfigure}
    \hfill
    \begin{subfigure}[b]{.45\linewidth}
        \includegraphics[width=\linewidth]{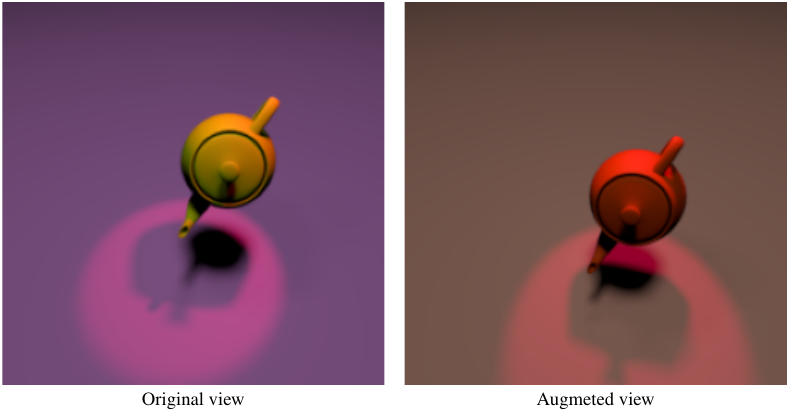}
        \caption{
        \textbf{Example Input}: \emph{3DIdent}}
        \label{fig:input_3dident}
    \end{subfigure}
\end{figure}

\looseness=-1\textbf{Implementation Details.}
We used the implementation from~\citep{michlo2021Disent} for Ada-GVAE~\citep{locatello2020weakly}, following the same architecture as \citep[][Tab. 1 Appendix ]{locatello2020weakly}. For our method, we use ResNet-18~\citep{he2015deep} as the image encoder, details given in~\cref{tab:encoders_cm3di}. For both approaches, we set \textsc{encoding size=10}, following the setup in~\citet{locatello2020weakly}. 

\subsection{Content-Style Identifiability on Images}
\label{app:c3di}

\looseness=-1\textbf{Datasets.}
\emph{Causal3DIdent}~\citep{von2021self} extends \emph{3Dident}~\citep{zimmermann2021contrastive} by introducing different classes of objects, 
thus \emph{object shape} (or \emph{class})
is added as an additional \emph{discrete} factor of variation. We extend the image pairs experiments from~\citep{von2021self} by inputting three views, as shown in~\cref{fig:input_causal3di}, where the second and third images are obtained by perturbing different subsets of latent factors of the first image. To perturb one specific latent component, we \emph{uniformly} sample one latent in the predefined latent space ($Unif[-1, 1]$, details see~\citep[][App. B]{von2021self}), then we use indexing search to retrieve the image in the dataset that has the closest latent values as the sampled ones. Note that only a finite number of images are available; thus, there is not always a perfect match. More frequently, we observe slight changes in the non-perturbing latent dimensions. For instance, the \emph{hues} of the third view is slightly different than the original view, although we intended to share the same \emph{hue} values.
\begin{figure}[htbp]
    \centering
    \begin{subfigure}[t]{0.48\linewidth}
        \includegraphics[width=.7\linewidth]{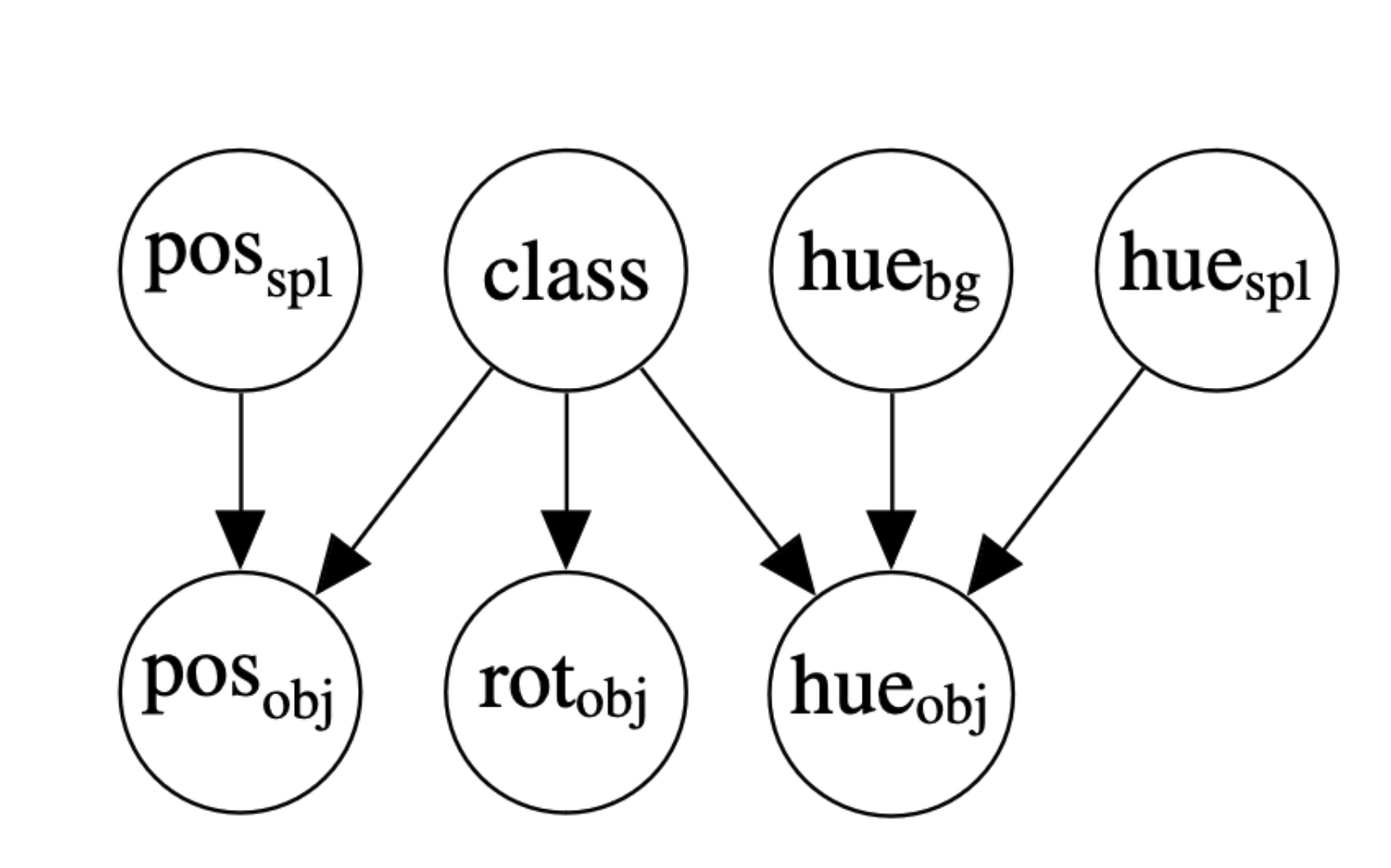}
        \caption{\textbf{Underlying causal relation} in \emph{Causal3DIdent} and \emph{Multimodal3DIdent} images. Figure adopted from~\citep[][Fig. 2]{von2021self}}
    \label{fig:graph_causal3di}
    \end{subfigure}
    \hfill
    \begin{subfigure}[t]{0.48\linewidth}
        \centering
    \includegraphics[width=\linewidth]{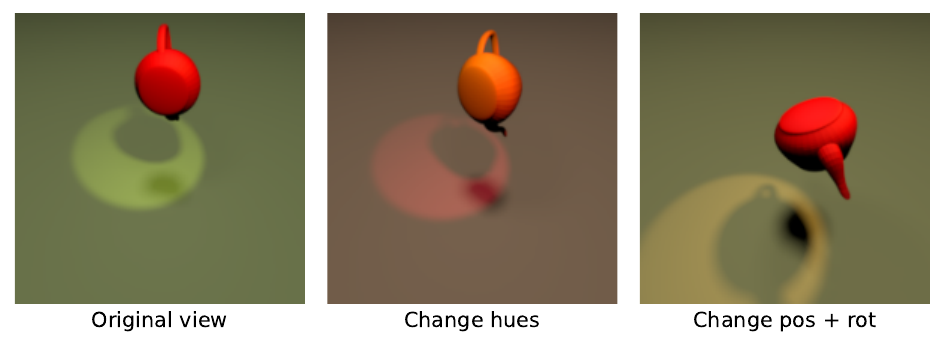}
    \caption{\textbf{Example Input}: \emph{Causal3DIdent}}
    \label{fig:input_causal3di}
    \end{subfigure}
    \caption{\textbf{\emph{Causal3DIdent}}: Underlying causal relations and input examples.}
\end{figure}

\looseness=-1\textbf{Implementation Details.}
The encoder structure and parameters are summarized in~\pcref{tab:encoders_cm3di,tab:param_cm3di}. We train using BarlowTwins~\citep{zbontar2021barlow} with cosine similarity and off-diagonal importance constant $\lambda = 1$. BarlowTwins is another contrastive loss that jointly encourages alignment and uniformity, by enforcing the cross-correlation matrix of the learned representations to be identity. 
The on-diagonal elements represent the \emph{\textcolor{Green}{content alignment}} while the off-diagonal elements approximate the \emph{\textcolor{Blue}{entropy regularization}}.

\looseness=-1\textbf{\cref{thm:ID_from_sets} Validation.}
We train \emph{content encoders}~\pcref{defn:content_encoders} on \emph{Causal3DIdent} to verify~\cref{thm:ID_from_sets}. Note that we experiment on three views~\citep{von2021self} cannot naively handle.
\Cref{table:causal3dident_full} summarizes results for all possible perturbations among the three views. 
We can observe that the discrete factor \emph{class} learned perfectly; dependent style variables become predictable from the content (\emph{class}) latent causal dependence. Note that this table shows similar results as in~\citep[][Table 6. Latent Transformation (LT)]{von2021self}.
We remark that there is a reality between theory in practice: In theory, we assume that the content variables share the \emph{exact same value} across all views; however, in practice, finding a perfect match of all of the \emph{continuous} content values become impossible, since there is only a finite number of training data available. We believe this reality gap negatively influenced the learning performance on the content variables, thus preventing efficient prediction on certain content variables, such as \emph{object hues}.

\begin{table}
\caption{\textbf{\cref{thm:ID_from_sets} Validation on} \emph{Causal3DIdent}: $R^2$ mean$\pm$std. \emph{\textcolor{Green}{Green}}: \textcolor{Green}{content}, \textbf{bold}: $R^2>0.50$.}
\label{table:causal3dident_full}
\resizebox{\linewidth}{!}{%
\begin{tabular}{c|c|cccc|ccc|ccc}
\toprule
\cellcolor{Gray!30}                                     & \cellcolor{Gray!30}                        & \multicolumn{4}{c|}{\cellcolor{Gray!30}positions}                                                                                                               & \multicolumn{3}{c|}{\cellcolor{Gray!30}hues}                                                               & \multicolumn{3}{c}{\cellcolor{Gray!30}rotations}                                                                      \\ \cline{3-12} 
\rowcolor{Gray!30} 
\multirow{-2}{*}{\cellcolor{Gray!30}\thead{\small{Views generated} \\ \small{by changing}}} & \multirow{-2}{*}{\cellcolor{Gray!30}class} &\cellcolor{Gray!30} $x$                                 &\cellcolor{Gray!30} $y$                                 &\cellcolor{Gray!30} $z$                        &\cellcolor{Gray!30} spotl                                 &\cellcolor{Gray!30} obj                           &\cellcolor{Gray!30} spotl                       &\cellcolor{Gray!30} bkg                                &\cellcolor{Gray!30} $\phi$                                            &\cellcolor{Gray!30} $\theta$                      &\cellcolor{Gray!30} $\psi$                        \\ \midrule
hues                                                  & {\color{Green} \textbf{1.00±0.00}}       & {\color{Green} \textbf{0.76±0.01}} & {\color{Green} \textbf{0.56±0.02}} & {\color{Green} 0.00±0.00} & {\color{Green} \textbf{0.82±0.01}} & 0.27±0.03                        & 0.00±0.01                        & 0.00±0.00                                 & \multicolumn{1}{l}{{\color{Green} 0.25±0.02}} & {\color{Green} 0.27±0.02} & {\color{Green} 0.27±0.02} \\
positions                                             & {\color{Green} \textbf{1.00±0.00}}       & 0.00±0.01                                 & 0.46±0.02                                 & 0.00±0.01                        & 0.00±0.01                                 & {\color{Green} 0.32±0.02} & {\color{Green} 0.00±0.01} & {\color{Green} \textbf{0.92±0.00}} & {\color{Green} 0.26±0.02}                     & {\color{Green} 0.29±0.02} & {\color{Green} 0.27±0.02} \\
rotations                                             & {\color{Green} \textbf{1.00±0.00}}       & {\color{Green} 0.11±0.01}          & {\color{Green} \textbf{0.50±0.02}} & {\color{Green} 0.00±0.00} & {\color{Green} 0.06±0.01}          & {\color{Green} 0.31±0.02} & {\color{Green} 0.00±0.01} & {\color{Green} \textbf{0.83±0.01}} & 0.25±0.01                                            & 0.27±0.02                        & 0.06±0.01                        \\
hues+pos                                             & {\color{Green} \textbf{1.00±0.00}}       & 0.00±0.00                                 & 0.20±0.02                                 & 0.00±0.01                        & 0.00±0.01                                 & 0.14±0.02                        & 0.00±0.00                        & 0.00±0.01                                 & {\color{Green} 0.07±0.01}                     & {\color{Green} 0.18±0.02} & {\color{Green} 0.12±0.02} \\
hues+rot                                              & {\color{Green} \textbf{1.00±0.00}}       & {\color{Green} 0.09±0.02}          & {\color{Green} 0.36±0.02}          & {\color{Green} 0.00±0.00} & {\color{Green} \textbf{0.51±0.01}} & 0.25±0.02                        & 0.00±0.01                        & 0.00±0.01                                 & 0.00±0.01                                            & 0.25±0.02                        & 0.25±0.01                        \\
pos+rot                                           & {\color{Green} \textbf{1.00±0.00}}       & 0.00±0.00                                 & 0.21±0.02                                 & 0.00±0.01                        & 0.00±0.00                                 & {\color{Green} 0.07±0.01} & {\color{Green} 0.00±0.01} & {\color{Green} 0.23±0.02}          & 0.05±0.01                                            & 0.20±0.02                        & 0.13±0.02                        \\
hues+pos+rot                                         & {\color{Green} \textbf{1.00±0.00}}       & 0.00±0.00                                 & 0.42±0.02                                 & 0.00±0.01                        & 0.00±0.00                                 & 0.25±0.02                        & 0.00±0.00                        & 0.00±0.00                                 & 0.01±0.01                                            & 0.26±0.02                        & 0.26±0.02          \\
\bottomrule
\end{tabular}}
\end{table}

\subsection{Multi-modal Content-Style Identifiability under Partial Observability}
\label{app:m3di}
\looseness=-1\textbf{Dataset.}
\emph{Multimodal3DIdent}~\citep{daunhawer2023identifiability} augments \emph{Causal3DIdent}~\citep{von2021self} with text annotations for each image view, and discretizes the \emph{objection positions} $(x, y, z)$ to categorical variables. 
In particular, \emph{object-zpos} is a constant and thus not shown in our evaluation~\pcref{fig:m3di_res}. 
Our experiment extends~\citep{daunhawer2023identifiability} by adding one additional image to the original image-text pair, 
perturbing the \emph{hues}, \emph{object rotations} and \emph{spotlight positions} of the original image (Uniformly sample from $Unif[0, 1]$). 
Thus, $(\mathrm{img}_0, \mathrm{img}_1)$ share \emph{object shape} and \emph{background color}; 
Thus, $(\mathrm{img}_0, \mathrm{txt}_0)$ share \emph{object shape} and \emph{object x-y positions}; 
both $(\mathrm{img}_1, \mathrm{txt}_0)$ and the joint set $(\mathrm{img}_0, \mathrm{img}_1, \mathrm{txt}_0)$ share only the \emph{object shape}. 
One example input is shown in~\cref{fig:input_m3di}.

\begin{wrapfigure}{r}{.63\linewidth}
     \centering
     \vspace*{-3pt}
     \subfloat{\includegraphics[width=0.33\textwidth]{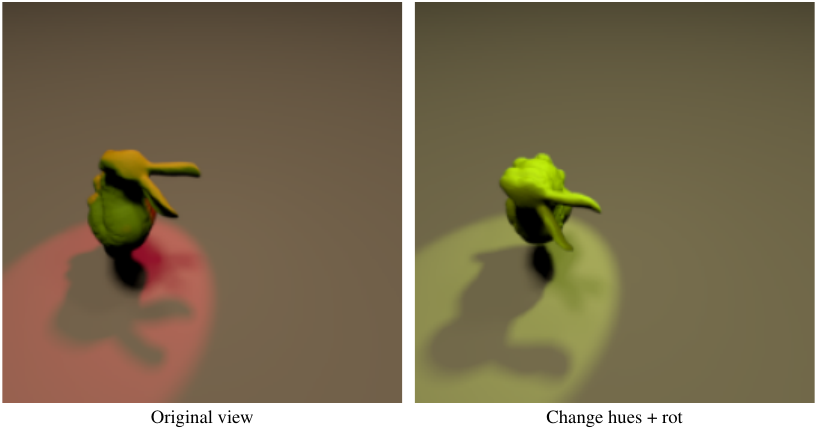}}%
    \hspace*{10pt} \begin{minipage}{.20\textwidth}
    \vskip-80pt \emph{Text description:} A \textbf{hare} of  \textbf{bright yellow green} color is visible, positioned at the \textbf{mid-left} of the image.
    \end{minipage}
    \caption{\textbf{Example input}:  \emph{Multimodal3DIdent}. \emph{Left:} pair of images that are original and perturbed images. \emph{Right:} Text annotation for the original view.}
    \vspace*{-10pt}
    \label{fig:input_m3di}
\end{wrapfigure}

\looseness=-1\textbf{Implementation Details.}
\Cref{tab:param_cm3di,tab:encoders_cm3di} shows the network architecture and implementation details, mostly following~\citep{daunhawer2023identifiability}. Note that we use the same encoding size for both image and text encoders for the convenience of implementation. We train using \emph{BarlowTwins} with $\lambda = 1$.
In practice, we treat the \emph{unknown content sizes} as a list of hyper-parameters and optimize it over different validations. 

\begin{table}
    \centering
    \caption{\textbf{Encoder Architectures} for \emph{Causal3DIdent} and \emph{Multimodal3DIdent}.}
    \label{tab:encoders_cm3di}
    \begin{tabular}{ll}
    \toprule
     \rowcolor{Gray!20}\textbf{Image Encoder} & \textbf{Text Encoder}\\
     \midrule
        \emph{Input size = H $\times$ W $\times$ 3}&\emph{Input size = vocab size}\\
       ResNet-18(hidden size)  & Linear(fbase, text embedding dim) \\
       LeakyReLu($\alpha=0.01$)  &  Conv2D(1, fbase, 4, 2, 1, bias=True)\\
       Linear(hidden size, image encoding size) & BatchNorm(fbase); ReLU\\
       & Conv2D(fbase, fbase$\cdot$2, 4, 2, 1, bias=True)\\
       & BatchNorm(fbase$\cdot$2); ReLU\\
       & Conv2D(fbase$\cdot$2, fbase$\cdot$4, 4, 2, 1, bias=True)\\
       & BatchNorm(fbase$\cdot$4); ReLU\\
       & Linear(fbase$\cdot$4$\cdot$3$\cdot$16, text encoding size)\\
    \bottomrule
    \end{tabular}
\end{table}

\looseness=-1 \textbf{Further Discussion about~\cref{subsec:case2}.} The fundamental difference between the \emph{Multimodal3D} and \emph{Causal3DIdent} datasets, as mentioned above, makes a direct comparison between our results in~\cref{fig:m3di_res} and~\citep{von2021self} harder. However, ~\Cref{table:causal3dident_full} shows similar $R^2$ scores as the results given in~\citet[][Sec 5.2]{von2021self}, which verifies the correctness of our method. 

\looseness=-1 \textbf{\cref{thm:ID_from_sets} Validation.}
We additionally learn \emph{content encoders} on three partially observed views $(\mathrm{img}_0, \mathrm{img}_1, \mathrm{txt}_0)$ using the loss from~\citet{zbontar2021barlow}, to justify~\cref{thm:ID_from_sets}. We use the same architecture and parameters as summarized in~\cref{tab:encoders_cm3di,tab:param_cm3di}.
~\Cref{tab:m3di_table} shows the content encoders consistently predict the content variables well and that our evaluation results highly align with~\citep[][Fig. 3]{daunhawer2023identifiability} on image-text pairs $(\mathrm{img}_0, \mathrm{txt}_0)$ as inputs, which empirically validates~\cref{thm:ID_from_sets}.

\begin{table}
\caption{\textbf{\cref{thm:ID_from_sets} Validation on} \emph{Multimodal3DIdent}: $R^2$ mean$\pm$std. \emph{\textcolor{Green}{Green}}: content, \textbf{bold}: $R^2 > 0.50$.}
\label{tab:m3di_table}
\centering
\begin{adjustbox}{max width=\linewidth}
\begin{tblr}{
  cells = {c},
  row{1} = {bg=Gray!30},
  row{2} = {bg=Gray!30},
  cell{1}{2} = {r=2}{},
  cell{1}{9} = {r=2}{},
  cell{1}{13} = {r=2}{},
  cell{1}{3} = {c=3}{},
  cell{1}{6} = {c=3}{},
  cell{1}{10} = {c=2}{},
  hline{2} = {1}{Gray!30,rightpos=-0.01},
  hline{2} = {12}{Gray!30,leftpos=-0.01,rightpos=-0.01},
  vline{2-3} = {1-2}{0.025em},
  vline{2,6,9,10,12,13} = {1-2}{0.025em},
  vline{2-3,6,9,10,12,13} = {3-4}{0.025em},
  hline{1} = {1}{-}{0.1em},
  hline{1} = {2}{-}{0.01em,fg=white},
  hline{2} = {3-11}{0.01em},
  hline{3} = {1}{-}{0.01em,fg=white},
  hline{3} = {2}{-}{0.06em},
  hline{3} = {3}{-}{0.01em,fg=white},
  hline{5} = {1}{-}{0.01em,fg=white},
  hline{5} = {2}{-}{0.1em},
}
views generated  & class                            & img pos                          &                                  &       & img hues      &       &     & txt class                        & txt pos                          &                                  & txt hue         & txt phrasing \\
                           by changing &                                  & $x$                           & $y$                           & spotl & obj           & spotl & bkg &                                  & $x$                           & $y$                           & obj\_color\_idx &              \\
hues + rot                 & \textcolor{Green}{\textbf{0.82 \textpm 0.01}} & \textcolor{Green}{\textbf{1.00 \textpm 0.00}} & \textcolor{Green}{\textbf{1.00 \textpm 0.00}} & 0.00 \textpm 0.00  & \textbf{0.87 \textpm 0.01} & 0.00 \textpm 0.00 & -& \textcolor{Green}{\textbf{0.85 \textpm 0.03}} & \textcolor{Green}{\textbf{1.00 \textpm 0.00}} & \textcolor{Green}{\textbf{1.00 \textpm 0.00}} &  0.15 \textpm 0.02          &  0.21 \textpm 0.02       
\\
pos   
&\textcolor{Green}{\textbf{1.00 \textpm 0.00}}  &0.47 \textpm 0.02
&\textbf{0.64 \textpm 0.01}
&0.00\textpm 0.00   
&\textbf{0.67 \textpm 0.02}
&0.00 \textpm 0.00
&- 
&\textcolor{Green}{\textbf{1.00 \textpm 0.00}}  &0.34 \textpm 0.02                            
&\textbf{0.94\textpm 0.01}                      & 0.16 \textpm 0.03
& 0.21 \textpm 0.02                
\end{tblr}
\end{adjustbox}
\end{table}

\subsection{Multi-Task Disentanglement with Sparse Classifiers}
Following~\cref{example:intuitive}, we synthetically generate the class labels by linear/nonlinear labeling functions on the shared content values, which resembles the underlying inductive bias from~\citep{lachapelle2022synergies,fumero2023leveraging}, that the shared features across different images with the same label should be most task-relevant. Here, we use the same sparse-classifier implementation from~\citep{fumero2023leveraging}. 
We remark that the goal of this experiment is to verify our expectation from~\cref{thm:ID_from_sets} that the method of~\citep{fumero2023leveraging} can be explained by our theory, although they assume mutually independent latents, which is a special case of our setup. In our experimental setup, an input gets a label $1$ when the following labeling function value is greater than zero:
\begin{itemize}%
    \item Linear: $\sum_{j = 1}^{d} {\hat{\zb}_j}$ where $d$ denotes the encoding size.
    \item Nonlinear: $\tanh\left(\sum_{j = 1}^{d} {\hat{\zb}_j}^3\right)$
\end{itemize}
Thus, we have resembled the inductive hypothesis in~\citet{fumero2023leveraging} that the classification task is \emph{only} dependent on the shared features. The fact that the binary classification is solved in several iterations verifies that~\citet{fumero2023leveraging} used the same \emph{soft} alignment principle as described in~\cref{thm:ID_from_sets}.

\section{Discussion}
\label{sec:extended_discussion}
\paragraph{The Theory-Practice Gap} It is noticeable that some of the technical assumptions we made in the theory may not exactly hold in practice. A common assumption in identifiability literature is that the latent variables $\zb$ are continuous, while this is not true for e.g. the \emph{object shape} in \emph{Causal3DIdent}~\citep{von2021self} and \emph{object shape, positions} in \emph{Multimodal3DIdent}~\citep{daunhawer2023identifiability}. Another related gap regarding the dataset is that the additional views are generated by uniformly sampling a subset of latents from the original view and then trying to retrieve an image among the \emph{existing} finite dataset, whose latent value is closest to the proposed one. However, having only a finite number of images implies that always finding a perfect match for each perturbed latents is almost impossible in practice. As a consequence, the designed to be strictly aligned content values between different views could differ from each other by a certain margin. 
Also, both~\cref{thm:ID_from_sets,thm:general_ID_from_sets_size_unknown} holds asymptotically and the global minimum is obtained only when given infinitely amount of data. Given that there is no closed-form solution for \emph{\textcolor{Blue}{entropy regularization}} term~\cref{eq:main_loss_set,eq:main_loss_set_size_unknown}, it is approximated either using negative samples~\citep{oord2019representation, chen2020simple} or by optimizing the cross-correlation between the encoded information to be close to Identity matrix~\citep{zbontar2021barlow}; in both cases there is only a finite number of samples available, which makes converging to global minimum almost impossible in practice.

\textbf{Discovering Hidden Causal Structure from Overlapping Marginals.} 
Identifying latent blocks $\{\zb_{B_i}\}$ provides us with access to the marginal distributions  over the corresponding subsets of latents $\{p(\zb_{B_i})\}$. With observed variables and known graph, this has been termed the ``causal marginal problem''~\citep{gresele2022causal}, and our setup could therefore also been seen as generalization along those dimensions.
It may be possible to extract some causal relations from the inferred marginal distributions over blocks, either by imposing additional assumptions or through constraint-based approaches~\citep{triantafillou2010learning}.

\textbf{How to Learn the Content Sizes?}
~\Cref{thm:general_ID_from_sets_size_unknown} shows that content blocks from any arbitrary subset of views can be discovered simultaneously using view-specific encoders~\pcref{defn:view_specific_encoders}, content selectors~\pcref{defn:content_selectors} 
and
some projections~\pcref{defn:aux_transformations}. We remarked in the main text that optimizing the information-sharing regularizer~\pcref{defn:info_sharing_reg} is highly non-convex and thus impractical. We proposed alternatives for both unsupervised and supervised cases: for self-supervised representation learning, one could employ \emph{Gumble-Softmax} to learn the hard mask. We hypothesize that if there is an additional inclusion relation about the content blocks available, for example, we know that $C_1 \subseteq C_2 \subseteq C_3$, then the learning process could be eased by coding this inclusion relation in the mask implementation. This additional information is naturally inherited from the fact that the more views we include, the smaller the shared content will be. 
Another idea would be manually allocating individual content blocks in the learned encoding in a sequential manner, e.g. we set index=1, 2, 3 for the first content block and index=4, 5 for the second content block, and enforcing alignment correspondingly. Thus, for each view, we learn a concatenated representation of the shared content. Although this method does not perfectly follow the theoretical setting in the~\cref{thm:general_ID_from_sets_size_unknown}, it still learn all of the contents simultaneously and shows faster convergence. 
In classification tasks, the hard alignment constraint is relaxed to some soft constraint within one equivalence class e.g. samples which have the same label. In this case,  we can replace the binary content selector with linear readouts, as studied and implemented by~\citet{fumero2023leveraging}. Another, yet the most common approach to deal with this problem is to treat the content sizes as hyperparameters, as shown in~\citep{von2021self,daunhawer2023identifiability}. 

\textbf{Trade-off between Invertibility and Feature Sharing.} 
An invertible encoder implies that the extracted representation is a lossless compression, which means that the original observation can be reconstructed using this learned representation, given enough capacity. On the one hand, the invertibility of the encoders is enforced by the \emph{\textcolor{Blue}{entropy regularization}}, such that the encoder preserves all information about the content variables; on the other hand, the info-sharing regularizer~\pcref{defn:info_sharing_reg} encourages reuse of the learned feature, which potentially prevents perfect reconstruction for each individual view. 
Intuitively,~\cref{thm:general_ID_from_sets_size_unknown} seeks the sweet spot between invertibility and feature sharing: When the encoder shares more than the ground truth, then it loses information about certain views, and thus the compression is not lossless; When the invertibility is fulfilled but the info-sharing is not maximized, then the learned encoder is not an optimal solution either, given by the regularization penalty from the infor-sharing regularizer.

\textbf{Causal Representation Learning from Interventional Data.}
Our framework considers purely \emph{observational} data, 
where multiple partial views are generated from concurrently sampled latent variables using view-specific mixing functions. 
Recent works~\citep{ liang2023causal,buchholz2023learning,von2023nonparametric} have shown identifiability in non-parametric causal representation learning using interventional data, 
allowing discovering (some) hidden causal relations. 
Since simultaneously identifying the latent representation and the underlying causal structure in a \emph{partially observable} setup has been a long standing goal, 
we believe incorporating interventional data into the proposed framework could be one interesting direction to explore.

%
%
%

\end{document}